\documentclass[11pt]{article}

\usepackage[margin=1in]{geometry}
\usepackage{microtype}
\usepackage{setspace}

\usepackage[style=numeric,backend=biber,natbib=true]{biblatex}
\DeclareBibliographyAlias{preprint}{unpublished}
\usepackage{graphicx}
\usepackage[section]{placeins}
\usepackage{float}

\usepackage{enumitem}
\usepackage{array}
\usepackage{adjustbox}
\usepackage{booktabs}
\usepackage{algorithm}
\usepackage{algorithmicx}
\usepackage{algpseudocode} 
\usepackage{subcaption}
\usepackage{amsmath}
\usepackage{amssymb}
\usepackage{amsfonts}
\usepackage{mathtools}
\usepackage{dsfont}
\usepackage{amsthm}
\usepackage{thmtools}
\usepackage{makecell}
\usepackage{tabularx}
\usepackage{ragged2e}
\newcommand{\keywords}[1]{\par\noindent\textbf{Keywords:} #1\par}

\newcommand{\uriLogo}{%
  \IfFileExists{logos/RI.png}{\includegraphics[height=3.0ex]{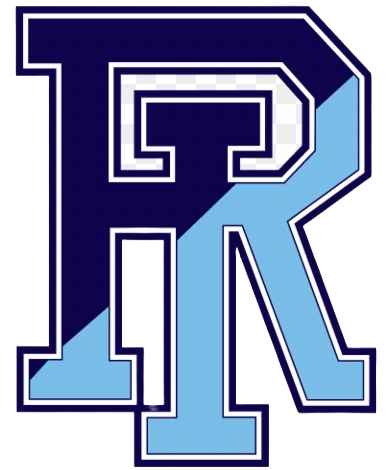}}{%
    \IfFileExists{logos/RI.png}{\includegraphics[height=3.0ex]{logos/RI.png}}{1}%
  }%
}
\newcommand{\hmsLogo}{%
  \IfFileExists{logos/HARVARD.png}{\includegraphics[height=3.5ex]{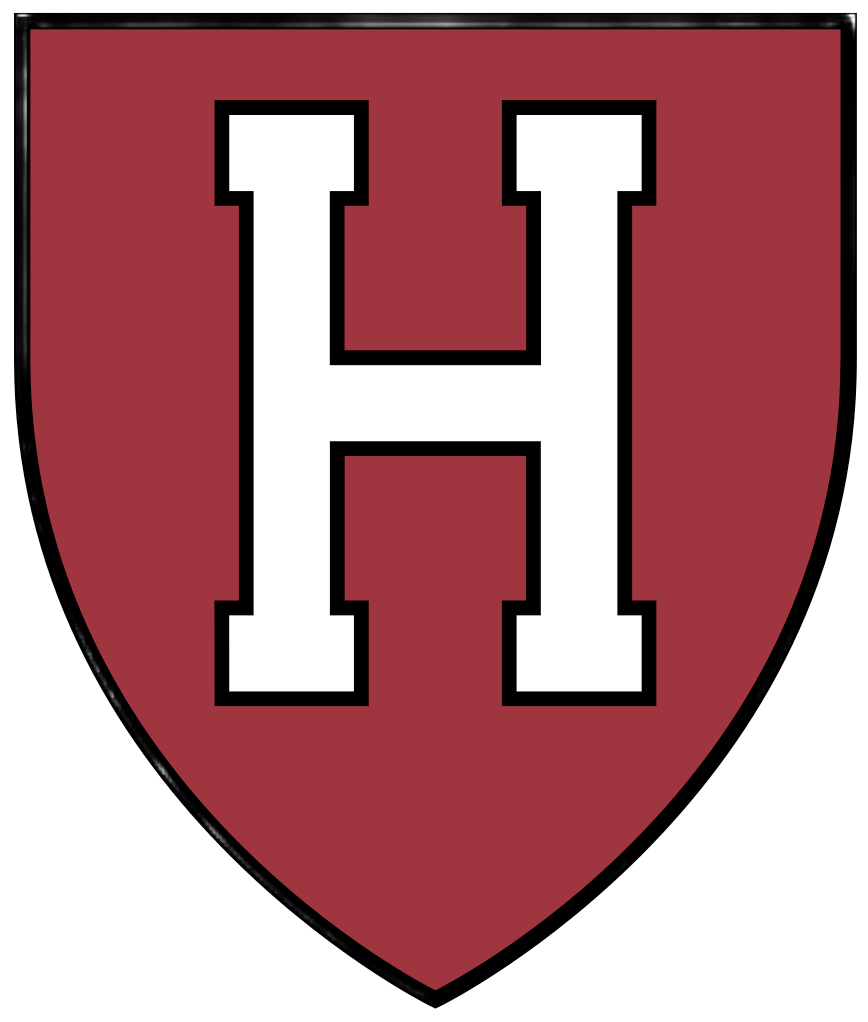}}{%
    \IfFileExists{logos/HARVARD.png}{\includegraphics[height=3.5ex]{logos/HARVARD.png}}{2}%
  }%
}

\newtheorem{theorem}{Theorem}

\newenvironment{proofsketch}{\noindent\textit{Proof sketch.}\ }{\hfill$\square$}

\renewcommand{\abstractname}{}

\DeclareMathOperator*{\argmax}{arg\,max}

\AtBeginDocument{%
  }
\begin{document}

\title{End-to-end Optimization of Belief and Policy Learning in Shared Autonomy Paradigms}

\author{
MH Farhadi, Ali Rabiee, Sima Ghafoori,\\
Anna Cetera, Andrew Fisher,
Reza Abiri \\
University of Rhode Island, Kingston, Rhode Island, USA
}
\date{}

\maketitle

\begin{abstract}
  Shared autonomy systems require principled methods for inferring user intent and determining appropriate assistance levels. This is a central challenge in human-robot interaction, where systems must be successful, while being mindful of user agency. Previous approaches relied on static blending ratios or separated goal inference from assistance arbitration, leading to suboptimal performance in unstructured environments. We introduce BRACE (Bayesian Reinforcement Assistance with Context Encoding), a novel framework that fine-tunes Bayesian intent inference and context-adaptive assistance through an architecture enabling end-to-end gradient flow between intent inference and assistance arbitration. Our pipeline processes the full Bayesian goal distribution rather than MAP estimates, conditioning collaborative control policies on environmental context and complete goal probability distributions. We provide analysis showing (1) optimal assistance levels should decrease with goal uncertainty and increase with environmental constraint severity, and (2) integrating belief information into policy learning yields a quadratic expected regret advantage over sequential approaches. We validated our algorithm against SOTA methods (IDA, DQN) using a three-part evaluation progressively isolating distinct challenges of end-effector control: (1) core human-interaction dynamics in a 2D human-in-the-loop cursor task, (2) non-linear physical dynamics of a robotic arm, and (3) integrated manipulation under goal ambiguity and environmental constraints. We demonstrate improvements over SOTA, achieving 6.3\% higher success rates and 41\% increased path efficiency, and 36.3\% success rate and 87\% path efficiency improvement over unassisted control. Results confirmed the advantage of integrated optimization is most pronounced in complex scenarios with goal ambiguity. Our method outperformed sequential approaches by 23\% in completion time in scenarios with high goal uncertainty, with success rate improvements of 13.1\% in complex multi-target environments. The demonstrated benefits of our integrated approach are generalizable across robotic domains requiring goal-directed assistance, advancing the SOTA for adaptive shared autonomy.
\end{abstract}

\keywords{Shared Autonomy, Human-Robot Interaction (HRI), Assistive Robotics, Collaborative Control, User Agency, Goal Uncertainty, Intent Inference, Belief-Conditioned Policy}
\section{Introduction}

Shared autonomy is a control paradigm in which a human and an automated system jointly operate a device, with the system adapting its assistance to the situation. For example, In assistive robotics for users with motor impairments, the system must infer whether the user intends to grasp a large water glass or a small pill bottle. It may provide minimal assistance during the reaching phase to preserve user agency, then increase support when aligning the gripper for a precision-critical grasp on the small bottle.

The core challenge in shared autonomy often comes from the inherent tension between two critical processes: inferring the user's goal (often a probabilistic inference problem) and determining the appropriate level of assistance (an optimization problem) \cite{jeon2020}. Prior approaches typically addressed this challenge as separate or sequential problems, either using fixed blending ratios regardless of context \cite{dragan2013policy}, separating goal inference from assistance arbitration \cite{javdani}, or solving them with model-based planning that requires known dynamics models \cite{jarrett}. Recent belief-conditioned methods such as SARI \cite{jonnavittula2022}, mutual-adaptation \cite{nikolaidis}, and IDA \cite{IDA} highlighted the need for an end-to-end integration, however, their findings either relied on Maximum a-posteriori (MAP) goals, binary interventions or failed to adapt optimally to the uncertainty in the user's goals.

Classical methods usually first estimate the user’s intent, then control the system using either a single best guess of that intent or a belief-space planner whose inference is fixed and does not adapt. BRACE instead learns assistance as $\gamma = f_\theta(s, b, c)$ using the goal belief $b$, and it links inference and control end-to-end so that the resulting performance (the downstream gradients) also shapes and updates the belief itself.

This work makes the following contributions:

\begin{enumerate}[leftmargin=*]
\item  
      \textbf{Full-belief conditioning:} Assistance is a function of the entire intent posterior, not the MAP goal, allowing a more nuanced reaction to user uncertainty. 

\item   
      \textbf{End-to-end coupling:} The intent inference module and assistance arbitration policy are optimized jointly from a single control objective; gradients from task returns reshape the belief to be decision-useful, avoiding the estimator–controller mismatch of sequential pipelines. 

\item  
      \textbf{Proof of advantage over regret bounds:} When different goals require different assistance (because intent is unclear and the environment is constrained), conditioning on belief and jointly optimizing inference and control reduces expected regret quadratically compared with MAP schemes. The benefit grows as constraints tighten and shrinks as belief entropy increases. That’s exactly when BRACE helps most.
\end{enumerate}
  
Recent belief-aware approaches such as SARI \cite{jonnavittula2022}, Mutual-Adaptation \cite{nikolaidis}, IDA \cite{IDA}, and STREAMS \cite{streams} have begun to couple intent inference with assistance. Our work differs from them in three main ways. \textbf{(i)} BRACE leverages the complete belief distribution $b\in\mathbb{R}^{|G|}$ rather than a single goal estimate or binary signal, which is critical in scenarios with high goal uncertainty. \textbf{(ii)} Their assistance logic is either huristic thresholds or secondary networks trained with the belief module frozen; BRACE performs joint gradient-based optimization of inference and control, giving provably lower regret (Theorem 2). \textbf{(iii)} None of the prior systems provide task-agnostic monotonicity guarantees or regret bounds.

The paper is organized as follows: Section 2 reviews prior work on intent inference and assistance arbitration. Section 3 introduces the BRACE framework, detailing its methodology, theoretical properties, and system components. Section 4 presents empirical results from multiple end-effector based control experiment scenarios, simulation benchmarks, and ablation studies. Section 5 discusses the findings and concludes with limitations and future work. The Appendices provide supplementary materials, including mathematical proofs, reward function details, and additional analyses.

\section{Related Works}

\subsection{Planning-based Approaches}

Javdani et al. \cite{javdani} modeled the problem as a Partially Observable Markov Decision Process (POMDP) with the human's goal as a hidden state. Their hindsight optimization algorithm provided computational tractability by approximating the POMDP solution, enabling assistance even when goal confidence was low. Jain and Argall \cite{jain2019} further extended probabilistic approaches for assistive robotics through recursive Bayesian filtering that fused observations like human behavior with different rationality levels. While these approaches offer principled formulations, they require simplifying assumptions about goal structures and rely on discrete optimization rather than learning adaptive strategies from interaction data.

In the planning literature, Monte Carlo Tree Search (MCTS) methods have been effectively applied to scenarios involving uncertainty in user preferences. Aronson et al. \cite{aronson} leveraged MCTS's statistical sampling approach to handle complex decision spaces without requiring exhaustive search. Nikolaidis et al. \cite{nikolaidis} advanced this field by formalizing human-robot mutual adaptation through Bounded-Memory Models embedded in partially observable decision processes, demonstrating performance improvements over one-way adaptive approaches. Complementing these works, Panagopoulos et al. \cite{panagopoulos2021} developed a Bayesian framework for intent recognition in remote navigation that fuses multiple observation sources, further extending probabilistic approaches to human-robot collaboration.

\subsection{Intention Inference and Reward Shaping}

Recent advances in reinforcement learning have enabled new approaches to shared autonomy that learn assistance policies. Reddy et al. \cite{DQN} implemented deep Q-learning to learn end-to-end assistive policies without explicit goal models, implicitly inferring intent through reward signals during human-in-the-loop training. Xie et al. \cite{xie2022} extended this work by developing a probabilistic policy blending approach that combines deep reinforcement learning with explicit uncertainty handling, demonstrating how different arbitration functions can blend human and agent actions for variable assistance levels.

Knox et al. \cite{knox} explored the interplay between reinforcement learning and human feedback, developing TAMER (Training an Agent Manually via Evaluative Reinforcement), which provides a framework for shaping agent behavior through human evaluative feedback. While not directly addressing shared autonomy, their work highlights how human input can guide learning systems, informing approaches to human-in-the-loop assistance.

Schaff and Walter \cite{schaff} developed residual policy learning that applies minimal corrections to human inputs to satisfy environmental constraints. This approach ensures humans remain mostly in control, which maintains user's sense of agency, but lacks a principled way to scale assistance when users struggle significantly. Liu et al. \cite{liu2023} proposed a complementary approach that blends imitation and reinforcement learning for robust policy improvement, addressing some of these limitations.

Accurate prediction of human intent forms a critical component of effective shared autonomy systems. Recent approaches to shared autonomy have employed Bayesian inference frameworks for probabilistic goal prediction, which have shown effectiveness in modeling uncertainty over user intentions \cite{jain2018, jain2019}. However, many current systems, including multimodal approaches, still rely on mapping certainty to assistance through predefined rules or thresholds rather than through learned policies that optimize for overall performance across varying contexts and user capabilities \cite{IDA, jonnavittula2022}. This separation between intent inference and assistance determination often limits the system's adaptability to complex scenarios with high goal uncertainty \cite{DQN}. Similar integration challenges exist in other robotic domains, \cite{goshtasbi} recently demonstrated that integrating AI with tactile sensing can significantly improve predictive accuracy by jointly processing multimodal sensory inputs, rather than treating sensing and interpretation as separate stages.

\begin{table}[h]
\caption{Comparison of BRACE with notable prior works}
\begin{adjustbox}{max width=\textwidth, center}
\renewcommand{\arraystretch}{1.3}  
\begin{tabular}{>{\raggedright\arraybackslash}p{2.5cm}
               >{\raggedright\arraybackslash}p{2.6cm}
               >{\raggedright\arraybackslash}p{2.6cm}
               >{\raggedright\arraybackslash}p{2.3cm}
               >{\raggedright\arraybackslash}p{2.4cm}}
\toprule
Framework & Belief input to controller & Assistance policy & Training & Validation \\
\midrule
Reddy et al.\ (2018) & Implicit (end-to-end) & Deep Q-learning & Human-in-the-loop RL & Lunar Lander, Quadrotor perching  \\[1pt]
Javdani et al.\ (2015) & Full POMDP belief & Hindsight optimization & Model-based & PR2 feeding + grasp tasks \\[1pt]
McMahan et al. (2024) & Binary intervene flag & Diffusion residual & Off-policy RL & Reacher, Lunar Lander\\[1pt]
Oh et al. (2021) & Heuristic confidence & Policy switch & Supervised & reach-and-grasp \\[1pt]
BRACE (ours) & Full belief vector & Continuous $\gamma$ via actor-critic & Joint end-to-end & Reacher-2D; Planar cursor control; pick and place \\
\bottomrule
\end{tabular}
\end{adjustbox}
\end{table}

\subsection{Sequential vs. Integrated Approaches}

Most existing systems employ sequential or parallel architectures that separate goal inference from assistance determination. Dragan and Srinivasa \cite{dragan2013policy} introduced policy-blending with linear interpolation between human and autonomous control, providing mathematical justification for the blending formulation but employing a fixed blending parameter.

Several actor-critic approaches treat goal probabilities as fixed inputs and learn only the blending policy. The most prominent example is The Disagreement method of \citet{oh2021disagreement}, implementing autonomy switching based on a predefined modality test over von Mises distributions. Our work addresses this gap in two ways: (i) we integrate intent inference within the same unified architecture, allowing for joint optimization of inference and control parameters, and (ii) we provide theoretical regret bounds for the integrated formulation, demonstrating quantifiable advantages over sequential approaches. Recent work by McMahan et al. \cite{IDA} introduced interventional diffusion assistance (IDA) that employs a diffusion model copilot and only intervenes when the expert's action is deemed superior to the human's.

We chose IDA and Reddy’s DQN because they are strong, widely used model-free baselines for shared autonomy and thus provide informative tests of BRACE. IDA is a recent method with demonstrated gains in simulation and user studies; it explicitly targets the autonomy–assistance trade-off and serves as a benchmark for learned arbitration. Reddy’s DQN is the standard end-to-end model-free RL baseline in the community; using it tests whether BRACE’s explicit belief conditioning offers benefits beyond an implicit intent learner.

\section{Methodology}

We model assistance allocation as a policy over the blend parameter. 
For state $s$, belief $b$ over goals, and context $c$, BRACE outputs $\gamma\in[0,1]$. Fig.~\ref{fig:sharedautonomy} shows the data flow: a Bayesian module updates $b$, while an actor–critic uses $(s,b,c)$ to set $\gamma$. 
Joint training allows performance gradients to shape both the belief and the arbitration policy.

\begin{figure}[htbp]
    \centering
    \includegraphics[width=0.75\linewidth]{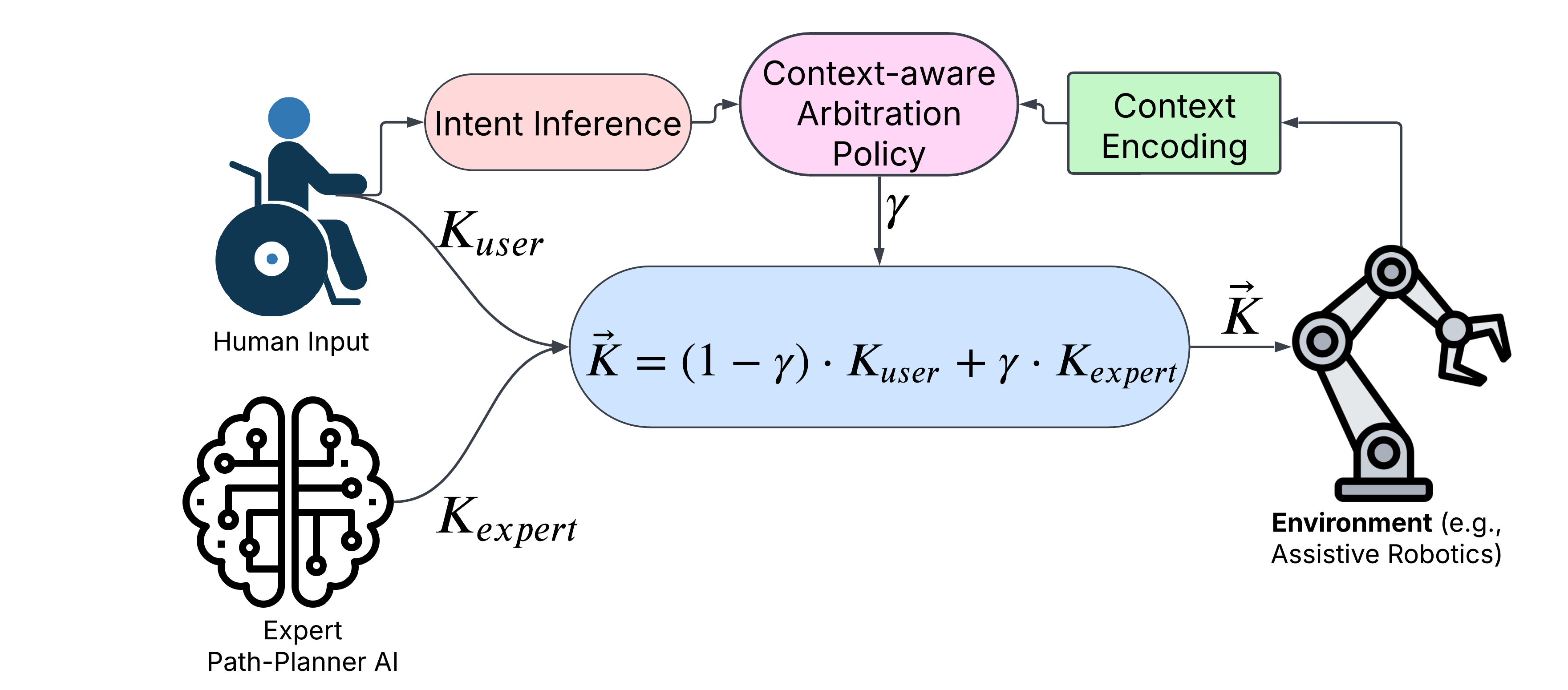}
    \caption{A general formulation of our shared autonomy framework: adaptive blending of human kinematic input ($K_{user}$) and expert policy's kinematic input ($K_{expert}$) via  $\gamma$ in Shared Autonomy frameworks.}
    \label{fig:sharedautonomy}
\end{figure}

The human's goal $g^* \in G$ is unknown to the system and must be inferred from observed inputs. We maintain a belief state $b_t(g)$ representing the probability distribution over potential goals at time $t$. The key challenge is effectively using this belief distribution to determine the optimal assistance level $\gamma$ that maximizes expected performance across goal uncertainty and environmental context.

We define $\gamma$ as $f_\theta(s, b_t, c(s))$ where $s$ is the current environment state, $b_t$ is the belief over goals, $c(s)$ represents environmental constraints (e.g., obstacle proximity), and $\theta$ are the parameters of function $f$. Recent works \cite{yousefi2025} have similarly explored adaptive reinforcement learning-based control using proximal policy optimization for shared autonomy systems, though without the explicit Bayesian belief integration we propose. This formulation can be viewed as a special case of a Partially Observable Markov Decision Process (POMDP) where the hidden state is the true goal, but we relax assumptions about known dynamics or predefined goal sets that characterized earlier approaches \cite{javdani}.

We implement our approach using a dual-head neural network architecture that processes state observations with specialized modules for goal inference and assistance policy optimization. As shown in Figure \ref{fig:architecture}, the network consists of two main components: a Bayesian inference module and a policy optimization Actor Critic network.

\begin{figure}[htbp]
    \centering
    \includegraphics[width=\linewidth]{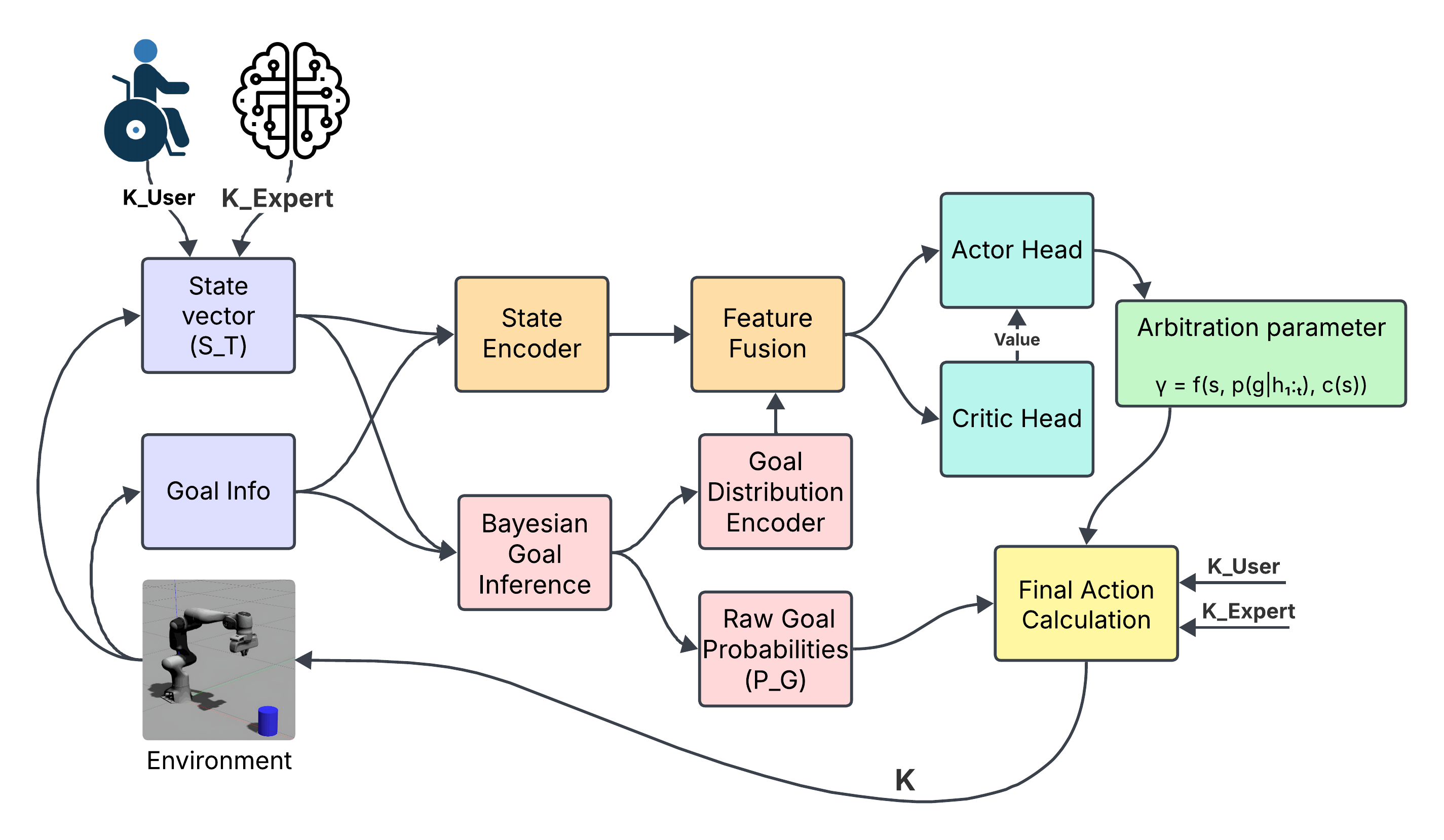}
    \caption{Architectural overview of our dual-head shared autonomy system. The network integrates a Bayesian goal inference module with a context-adaptive assistance policy learned through actor-critic reinforcement learning. The Bayesian inference module calculates belief states which are passed to the policy module to condition assistance decisions.}
    \label{fig:architecture}
\end{figure}

This architecture enables the policy to utilize the complete belief distribution—including entropy and multi-modal probability structures—rather than relying solely on MAP estimates. Empirical validation in Section 4.3 confirms this advantage, particularly in high-uncertainty scenarios.

\subsection{Training Procedure}
To leverage the complete belief distribution and enable joint optimization of both inference and control modules, we employ an integrated training procedure. Algorithm~\ref{alg:training} details our approach, which combines Bayesian inference with reinforcement learning through a curriculum-based training strategy.

\begin{algorithm}
\caption{BRACE Training Algorithm}
\label{alg:training}
\begin{algorithmic}[1]
\State Initialize parameters for Bayesian inference module and policy module
\State Pretrain Bayesian inference module
\State Initialize training mode (baseline or end-to-end)
\For{each curriculum stage}
    \State Initialize environment seed with appropriate difficulty
    \For{episode = 1 to N}
        \State Sample goal configuration and reset environment
        \State Initialize belief state $b_0$ with uniform prior over goals
        \For{t = 0 to T}
            \State Observe state $s_t$ and simluated human action $h_t$
            \State Update belief $b_t$ using Bayesian inference module
            \State Compute assistance level $\gamma_t$ and value estimate $V(s_t)$ from policy module
            \State Execute action $a_t = (1-\gamma_t)h_t + \gamma_t w_t$
            \State Observe reward $r_t$ and next state $s_{t+1}$
            \State Store transition $(s_t, b_t, \gamma_t, r_t, s_{t+1}, b_{t+1})$
        \EndFor
        \State Compute advantages and returns
        \State Update policy parameters using PPO objective
        \State Compute combined supervised and reinforcement loss
        \State Update inference module parameters through REINFORCE
        \State Apply training stability mechanisms
    \EndFor
    \State Advance to next curriculum setting
\EndFor
\end{algorithmic}
\end{algorithm}

The training procedure integrates both the inference and control aspects of our system. Lines 10-15 represent the core interaction loop, where the belief state is updated based on human actions and the assistance policy is executed. Gradients flow from the policy outputs back to the inference module parameters (line 19-20), enabling the Bayesian module to adapt based on downstream control performance. Finally, curriculum learning (lines 4-23) progressively increases task difficulty, improving learning efficiency for complex scenarios.

Unlike sequential approaches that separate goal inference from assistance arbitration, our joint optimization enables the full belief distribution to influence control decisions while simultaneously adapting inference parameters based on control outcomes.

\subsection{Theoretical Properties of Assistance Policy}

We denote $b$ for the intent belief, $c$ for constraint severity, $U_g(\gamma;s,c)$ for the goal-conditioned utility, $\overline U(\gamma;s,b,c)=\mathbb{E}_{g\sim b}[U_g(\gamma;s,c)]$, and
\[
\gamma^\star(s,b,c) \;\;=\;\; \arg\max_{\gamma\in[0,1]} \overline U(\gamma;s,b,c).
\]
Assumptions: (A1) $U_g(\gamma)$ is twice continuously differentiable and strongly concave in $\gamma$ near $\gamma^\star$; (A2) mixed partials $\partial^2 U_g/\partial \gamma \partial c \ge 0$ (assistance becomes more valuable as constraints tighten); (A3) we can reparameterize belief concentration by $\lambda$ with $dH(b_\lambda)/d\lambda<0$.

\begin{theorem}[Uncertainty and constraint monotonicity]
Under (A1)–(A3), the optimal assistance satisfies
\[
\frac{\partial \gamma^\star}{\partial H(b)} \;<\; 0
\qquad\text{and}\qquad
\frac{\partial \gamma^\star}{\partial c} \;>\; 0.
\]
\end{theorem}
\begin{proofsketch}
Let $F(\gamma,\lambda,c)=\partial_\gamma \overline U(\gamma; b_\lambda,c)$ with $F(\gamma^\star,\lambda,c)=0$. By the implicit function theorem,
\[
\frac{d\gamma^\star}{d\lambda}
= -\frac{\partial_{\lambda}\partial_{\gamma}\overline U(\gamma^\star; b_\lambda,c)}
{\partial_{\gamma}^2 \overline U(\gamma^\star; b_\lambda,c)}.
\]
The denominator is negative by strong concavity. Increasing $\lambda$ concentrates $b_\lambda$, aligning mass with goalwise optima and increasing $\partial_{\lambda}\partial_{\gamma}\overline U>0$, hence $d\gamma^\star/d\lambda>0$ and $d\gamma^\star/dH<0$. For $c$, monotone comparative statics with $\partial^2 U_g/\partial\gamma\partial c\ge 0$ yields $\partial \gamma^\star/\partial c>0$. Full details in App.~A.
\end{proofsketch}

\begin{theorem}[Integrated optimization advantage]
Let $\gamma_g^\star=\arg\max_{\gamma} U_g(\gamma;s,c)$. For a sequential scheme that chooses $\hat g=\arg\max b$ and uses $\gamma_{\hat g}^\star$, the expected regret gap versus belief-conditioned $\gamma^\star$ satisfies
\[
\mathbb{E}_{g\sim b}\!\big[U_g(\gamma_g^\star)-U_g(\gamma^\star)\big]
\;=\;\tfrac{1}{2}\,\mathbb{E}_{g\sim b}\!\big[\,|U_g''(\tilde\gamma)|\,(\gamma_g^\star-\gamma^\star)^2\big],
\]
for some $\tilde\gamma$ between $\gamma^\star$ and $\gamma_g^\star$. The gap vanishes if all goals share the same optimal assistance $\gamma_g^\star$.
\end{theorem}
\begin{proofsketch}
Second-order Taylor expansion of $U_g(\gamma)$ around $\gamma^\star$ (or $\gamma_g^\star$) gives
\[
U_g(\gamma_g^\star)-U_g(\gamma^\star)
= \tfrac{1}{2} |U_g''(\tilde\gamma)|\,(\gamma_g^\star-\gamma^\star)^2.
\]
Taking $\mathbb{E}_{g\sim b}$ yields the stated identity, showing a \emph{quadratic} advantage that grows with dispersion of $\{\gamma_g^\star\}$ and curvature $|U_g''|$. If $\gamma_g^\star$ is constant in $g$, the term is zero. Full proof and regularity conditions in App.~A.
\end{proofsketch}

Classical Bayesian adaptive control treats inference as an external estimator and optimizes control on a point estimate or a fixed belief update. Here  intent inference is jointly optimized with control; Theorems above formalize the resulting uncertainty/constraint monotonicity and the quadratic regret gap. Appendix~A contains extended proofs for the theorems.

\subsection{Bayesian Inference Module}

The Bayesian inference module implements recursive Bayesian filtering to maintain and update a probability distribution over potential goals. For a set of potential goals $G = \{g_1, g_2, ..., g_N\}$, the model updates the posterior probability of each goal $g_i$ given the trajectory of agent positions $X = \{x_1, x_2, ..., x_t\}$ and human inputs $H = \{h_1, h_2, ..., h_t\}$:

\begin{equation}
P(g_i|X,H) \propto P(g_i) \prod_{j=1}^{t} P(h_j|x_j,g_i)
\end{equation}

The likelihood function $P(h_j|x_j,g_i)$ models the human as a noisy-rational agent, with higher probabilities assigned to inputs that move the agent toward the goal:

\begin{equation}
P(h_j|x_j,g_i) \propto \exp(-\beta \cdot \text{cost}(h_j|x_j,g_i))
\end{equation}

where $\text{cost}(h_j|x_j,g_i)$ measures the deviation of the human's input from the optimal action for reaching goal $g_i$ from position $x_j$, and $\beta$ is a rationality parameter determining how closely the human follows the optimal policy. The cost function combines angular deviation and distance deviation as detailed in the supplementary materials.
The angular deviation $\theta_{dev}$ in our cost function inside the Bayesian module is calculated as the absolute angle between the human input vector $\vec{h}_j$ and the optimal direction vector from position $x_j$ to goal $g_i$:

\begin{equation}
\theta_{dev} = \left|\arccos\left(\frac{\vec{h}_j \cdot \vec{v}_{x_j\rightarrow g_i}}{|\vec{h}_j|\cdot|\vec{v}_{x_j\rightarrow g_i}|}\right)\right|
\end{equation}

where $\vec{v}_{x_j\rightarrow g_i}$ is the unit vector pointing from position $x_j$ to goal $g_i$.

The distance deviation $d_{dev}$ measures input magnitude deviation from optimal magnitude:

\begin{equation}
d_{dev} = \left|1 - \frac{|\vec{h}_j|}{h_{opt}}\right|
\end{equation}

where $h_{opt}$ is the optimal input magnitude given the distance to the goal, modeled as a function that decreases as the agent approaches the goal.

The cost function is a weighted combination of these deviations:

\begin{equation}
\text{cost}(h_j|x_j,g_i) = w_\theta \cdot \theta_{dev} + w_d \cdot d_{dev}
\end{equation}

where $w_\theta = 0.7$ and $w_d = 0.3$ are weighting parameters.

The Bayesian goal inference model was pre-calibrated using a dataset of 1968 cursor trajectories with known goal labels. We employed two-stage parameter estimation: a coarse grid search followed by Bayesian optimization. Optimal parameters were determined as: $\beta = 10.0$, $w_\theta = 0.7$, $w_d = 0.3$, with uniform priors over goals. To enhance robustness, we implemented temporal smoothing (goal probabilities smoothed using exponential moving average with decay rate $\alpha = 0.85$) and confidence calibration (isotonic regression on a validation set).
The warm started Bayesian model achieved 74.6\% accuracy at 25\% path completion, 91.2\% accuracy at 50\% path completion, and 98.7\% accuracy at 75\% path completion on a held-out test set of 194 trajectories.

\subsection{Assistance arbitration policy}
The actor-critic network consists of two 256-neuron hidden layers using ReLU activations—an actor network that outputs the continuous blending parameter and a parallel critic network for value estimation. This structure processes a 10-dimensional vector comprising state, human input, goal positions, and contextual distances, enabling context-sensitive assistance modulation. The policy is trained using Proximal Policy Optimization (PPO) \cite{schulman2017proximal} to maximize a reward function that balances several key components:

\begin{equation}
R = w_c R_c + w_p R_p + w_g R_g + w_a R_a + w_{gl} R_{gl}
\end{equation}

where each term represents a specific aspect of the shared control objective. Reward weights and network hyper-parameters are tuned via random search; The complete reward function with all terms and weights is detailed in Appendix~\ref{sec:detailed_reward}.
For assistance optimization, our PPO component includes an input layer combining features from the observation space and current goal probability distribution. The architecture uses two hidden layers with 256 neurons each and ReLU activation, followed by an output layer with a single neuron with tanh activation, mapping to $[-1,1]$ (rescaled to $[0,1]$ to represent $\gamma$). A separate value network with identical structure outputs a value estimate, The architecture of this PPO module is shown in Fig.~\ref{fig:actor_critic_arch}

\begin{figure}[htbp]
    \centering
    \includegraphics[width=\linewidth]{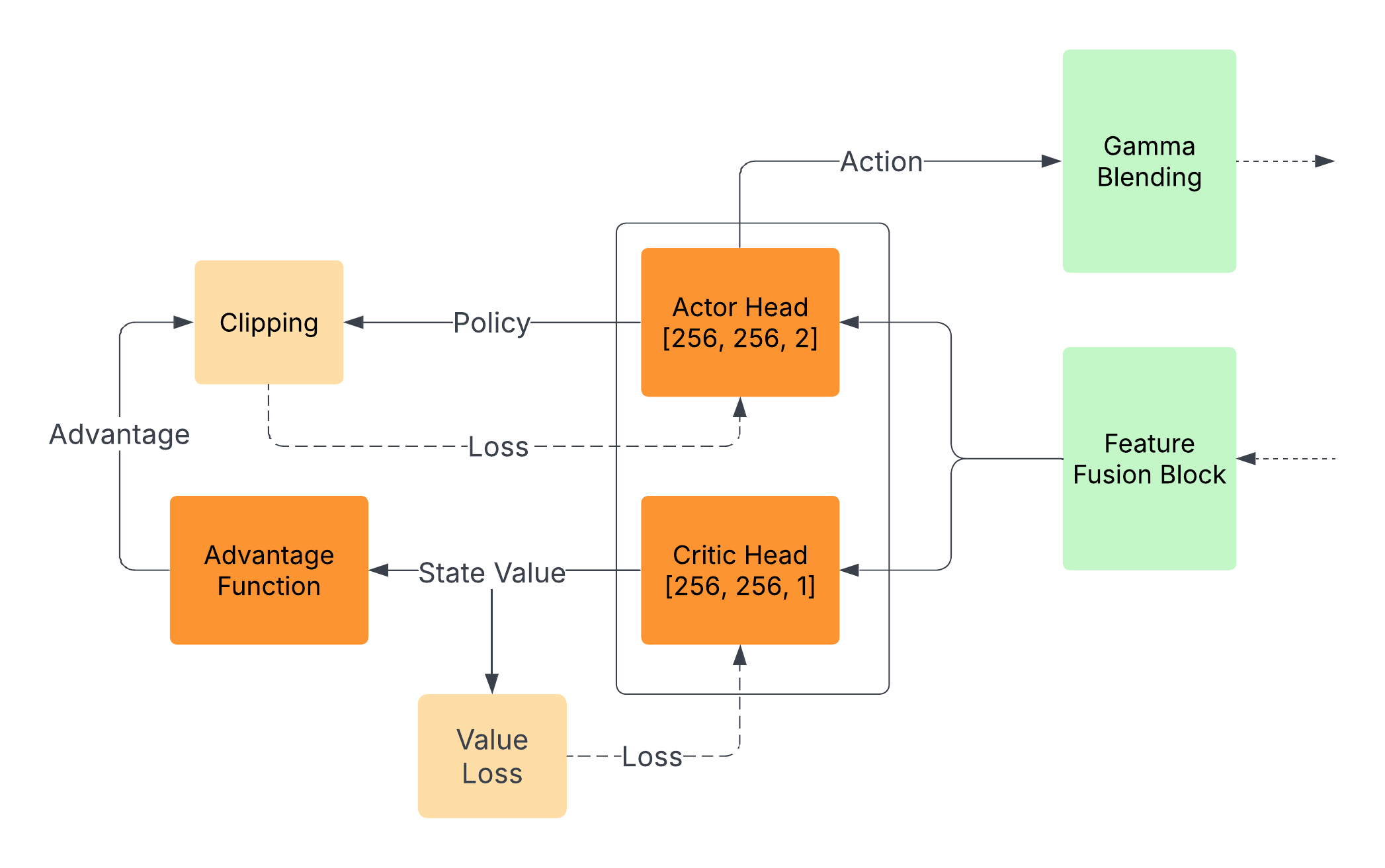}
    \caption{Architecture of the assistance arbitration PPO. The policy module consists of an Actor Head generating the blending parameter $\gamma$ and a Critic Head estimating state value. Both share a Feature Fusion Block that processes combined state and belief information.}
    \label{fig:actor_critic_arch}
\end{figure}

Key hyperparameters include learning rate of $3 \times 10^{-4}$ with cosine annealing, batch size of 1024 samples, 4 epochs per batch, discount factor of 0.99, GAE parameter of 0.95, and clip ratio of 0.2.

\subsection{Expert Policy}
\label{sec:appendix_expert_policy}

Our BRACE framework relies on an expert policy that acts as an oracle and provides perfect-path guidance to the assistance module. While the assistance arbitration module learns to blend human and AI control, the expert policy focuses on optimal goal-directed trajectories regardless of human input.

For the Planar cursor control task, the expert architecture consisted of a 4-layer network with [256, 256, 128, 64] neurons using ReLU activations and layer normalization.  We sampled 128 configurations across the hyperparameter space, optimizing learning rate ($3.2 \times 10^{-4}$), entropy coefficient (0.01), GAE parameter (0.92), and discount factor (0.99). To ensure reproducible optimization, we employed fixed seeds and evaluated each candidate configuration across 10 randomized environments, selecting the configuration that achieved the highest mean reward with lowest variance. As shown in Table~\ref{tab:adaptation}, BRACE maintained consistently high performance (within 8\% of optimal) even when paired with experts whose performance degraded by up to 37\%.

The expert's reward function balances goal progress, smoothness, and safety through:

\begin{equation}
R_{expert}(s, a) = 3.0 \cdot \frac{d_{t-1} - d_t}{d_{max}} - 0.8 \cdot \|\Delta\theta\|^2 - 2.5 \cdot \exp\left(-\frac{\min_{o \in \mathcal{O}} \|s - o\|}{d_{safe}}\right)
\end{equation}

\noindent where $d_t$ is distance to the nearest goal, $\|\Delta\theta\|^2$ penalizes angular acceleration for smooth trajectories, and the exponential term creates repulsive force around obstacles $\mathcal{O}$. This formulation ensures optimal path generation while maintaining safety.

During online operation, the expert's policy outputs represent the optimal action at each state, providing a perfect benchmark against which human inputs can be compared. The quality of this expert policy directly impacts BRACE's ability to provide appropriate assistance, particularly in high-gamma scenarios where the system relies heavily on the expert's guidance.

To demonstrate that BRACE genuinely adapts to user input rather than simply copying the expert policy, we conducted a series of controlled experiments using a hindsight optimization expert \cite{javdani}  as a benchmark. We varied this expert's optimality by introducing artificial constraints that limited its planning horizon and adding simulated delays to its reactions, creating a spectrum of expert policies ranging from optimal to significantly sub-optimal. Our results revealed that BRACE maintained consistently high performance (within 8\% of optimal) even when paired with experts whose performance degraded by up to 37\%. This resilience indicates true adaptation rather than imitation, as a simple copying approach would show proportional performance degradation. 

\begin{table}[h]
\centering
\caption{BRACE Performance with Degraded Expert Policies. The Expert Performance column shows the standalone success rate of expert policies that were artificially constrained with limited planning horizons or simulated delays. BRACE Performance measures the final system's success rate when paired with that specific expert.}
\label{tab:adaptation}
\begin{tabular}{lrrr}
\toprule
\textbf{Expert Policy} & \textbf{Expert Performance} & \textbf{BRACE Performance} & \textbf{Performance Delta} \\
\midrule
Full & 100\% & 98.2\% & -1.8\% \\
Horizon-Limited & 82.4\% & 94.1\% & +11.7\% \\
Delayed & 73.6\% & 92.8\% & +19.2\% \\
Random-Perturbed & 63.1\% & 90.5\% & +27.4\% \\
\bottomrule
\end{tabular}
\end{table}

\section{Results}
\label{sec:experiments}

We designed a three-part evaluation to validate BRACE against various challenges of human-in-the-loop robotic end-effector control, with each experiment isolating a different aspect of the end-effector assistance problem.

(1) The Human-Interaction Challenge (2D Cursor Control): First, we use a simulated 2D cursor-control task as a direct analogy for 2D end-effector positioning (e.g., in screen-based assistive tasks or simplified pick-and-place). This experiment's primary purpose is to isolate and test the core human-robot interaction dynamics. It validates BRACE against the noisy, variable, and unpredictable trajectories of real human users, which is especially critical for modeling assistive control from imperfect biosignals.

(2) The Physical-Dynamics Challenge (Reacher-2D): Then, we advance from simple point-mass physics to address the challenge of more complex physical dynamics. A real-world end-effector is part of a multi-joint kinematic chain. The Reacher-2D benchmark allows us to isolate this challenge, proving that BRACE's arbitration framework is robust enough to handle the non-linear, coupled dynamics common in robotic arm control.

(3) The Integrated 3D Task-Context Challenge (Fetch Pick-and-Place): Finally, we integrate these challenges into a high-dimensional manipulation task. A modified FetchPickAndPlace-v3 environment serves as our testbed for scaling BRACE to an end-effector operating in 3D space. Its purpose is to validate that our framework can jointly reason over the two most critical components of a realistic assistive task: high goal ambiguity (e.g., multiple similar target bins) and localized safety demands (e.g., navigating nearby obstacles during grasp and release).

\subsection{Planar Human-in-the-loop Cursor Control}

Our experiment employed a within-subjects design ($N=12$) where participants performed a goal-directed cursor control task using a DualSense controller. This study compared five different assistance conditions. All participants were required to complete informed consent forms prior to participation, screened for prior neurological diseases that could affect motor performance, and were compensated with a gift card. The study protocol received full approval from our Institutional Review Board (IRB) prior to participant recruitment and data collection. Participants completed all conditions to enable direct comparison of performance metrics across different assistance modes. An a priori power analysis using G*Power (v3.1) with parameters $\alpha=0.05$, power $(1-\beta)=0.8$, and an expected effect size of $f=0.35$ (medium-large) based on pilot testing indicated a minimum sample of 10 participants would be sufficient to detect significant differences. 

\begin{figure}[htbp]
    \centering
    \includegraphics[width=0.8\linewidth]{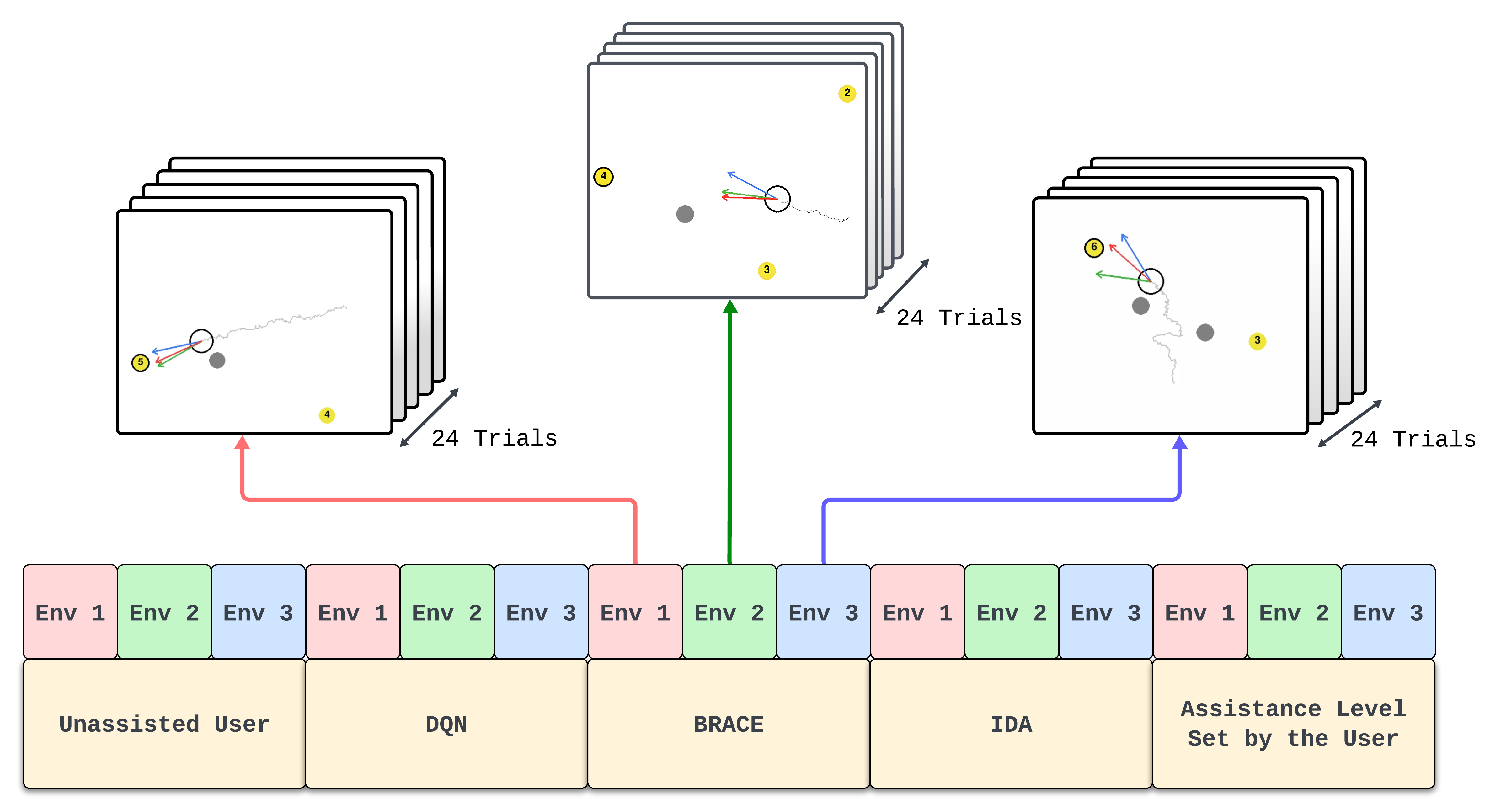}
    \caption{Structure of the human-in-the-loop cursor control user study, showing the distribution of trials across three test environments and five assistance conditions.}
    \label{fig:framework}
\end{figure}  

To create controlled testing environments, we first implemented a procedural environment generation approach, then chose three fixed random seeds that produced consistent scenarios across all participants and conditions. In each environment, goals were placed randomly within margins while maintaining minimum separation. Obstacles were strategically positioned along paths between the start position and goals with perpendicular offsets. This approach created three distinct environments that systematically varied in goal ambiguity (multiple targets in similar directions), constraint severity (narrow passages between obstacles), and path complexity (indirect routes to targets). Each participant completed 24 trials per environment (72 total) across all five assistance conditions. This trial count was determined through Monte Carlo power analysis to achieve 95\% confidence at 80\% power while minimizing fatigue, as individual trials lasted 5-15 seconds.

We used a balanced Latin square design to counterbalance condition and environment presentation order, minimizing potential sequence effects. Performance was measured through success rate, completion time, path efficiency, and throughput. Statistical analysis revealed strong effect sizes with 95\% confidence intervals across all primary metrics: success rate ($d=1.28$, 95\% CI [0.94, 1.62]), completion time ($d=1.41$, 95\% CI [1.07, 1.75]), and path efficiency ($d=1.35$, 95\% CI [1.01, 1.69]). These values indicate large practical significance in the improvements conferred by our approach. To account for individual differences, we employed stratified sampling across participant skill levels and conducted subgroup analyses. While inter-subject variance was noticeable, particularly in the manual condition, the improvement pattern remained consistent across all skill quartiles, with even the lowest-performing participants showing significant benefits from context-adaptive assistance (12.3\% improvement, $p<0.01$).

Participants completed target acquisition tasks under five control conditions across three environments with varied obstacle and goal configurations, performing 24 trials per environment for a total of 4,320 collected trajectories. We compared five control conditions: manual control (no assistance, $\gamma = 0$), Reddy et al.'s DQN approach \cite{DQN}, McMahan et al 's IDA approach \cite{IDA}, user-controlled gamma (where participants manually adjusted assistance level), and our approach BRACE, as illustrated in Figure~\ref{fig:framework}. Performance was evaluated using success rate, completion time, path efficiency, and throughput.

\subsubsection{Quantitative Performance Results}

Table \ref{tab:performance} summarizes the key performance metrics across all conditions.

BRACE outperformed all baseline methods across all metrics. Compared to manual control, BRACE achieved 36.3\% higher success rate (F(4,48) = 27.3, p < 0.001) and 60.9\% reduction in completion time (F(4,48) = 19.8, p < 0.001). BRACE improved over state-of-the-art with 6\% higher success rate, 42\% better path efficiency and 32\% faster completion times. Even compared to user-controlled assistance levels, which performed well due to explicit human modulation of assistance levels, our automatic approach achieved 8.5\% better path efficiency and 26\% faster completion times.

Figure \ref{fig:gamma_comparison} confirms our theoretical prediction that optimal assistance varies with environmental context and goal certainty. High assistance values (yellow regions, $\gamma > 0.8$) appear near obstacles and in narrow passages, while low assistance (blue, $\gamma < 0.3$) occurs in open spaces and areas with goal ambiguity. Temporal analysis shows that assistance starts at 0.28 $\pm$ 0.12 during trajectory initiation when uncertainty is highest, then increases to 0.74 $\pm$ 0.09 as goal inference confidence improves. This dynamic modulation provides advantages over both DQN approaches that use embedded policies and binary intervention methods like IDA, while achieving better consistency than manual user-controlled assistance.

\begin{table}[h]
\caption{Performance metrics across control conditions (Mean $\pm$ (SD across participants))}
\label{tab:performance}
\centering
\begin{tabular}{lccccc}
\toprule
& No assist & DQN & BRACE & IDA & Manual gamma level \\
\midrule
Success (\%) & 72.1 $\pm$ 3.2& 89.8 $\pm$ 2.5& 98.3 $\pm$ 1.7 & 92.5 $\pm$ 2.4& 86.8 $\pm$ 1.9\\
Time (s) & 8.44 $\pm$ 0.31 & 5.62 $\pm$ 0.26& 3.30 $\pm$ 0.22 & 5.07 $\pm$ 0.27& 4.5$\pm$ 0.23\\
Path Eff. & 0.43 $\pm$ 0.05 & 0.75 $\pm$ 0.05 & 0.89 $\pm$ 0.06 & 0.63 $\pm$ 0.06& 0.82 $\pm$ 0.06 \\
Throughput & 1.14 $\pm$ 0.10 & 1.2 $\pm$ 0.10& 1.27 $\pm$ 0.08 & 1.19$\pm$ 0.10& 1.25 $\pm$ 0.09 \\
\bottomrule
\end{tabular}
\end{table}

\begin{figure}[H]
    \centering
    \begin{subfigure}[b]{0.48\textwidth}
        \centering
        \includegraphics[width=\textwidth]{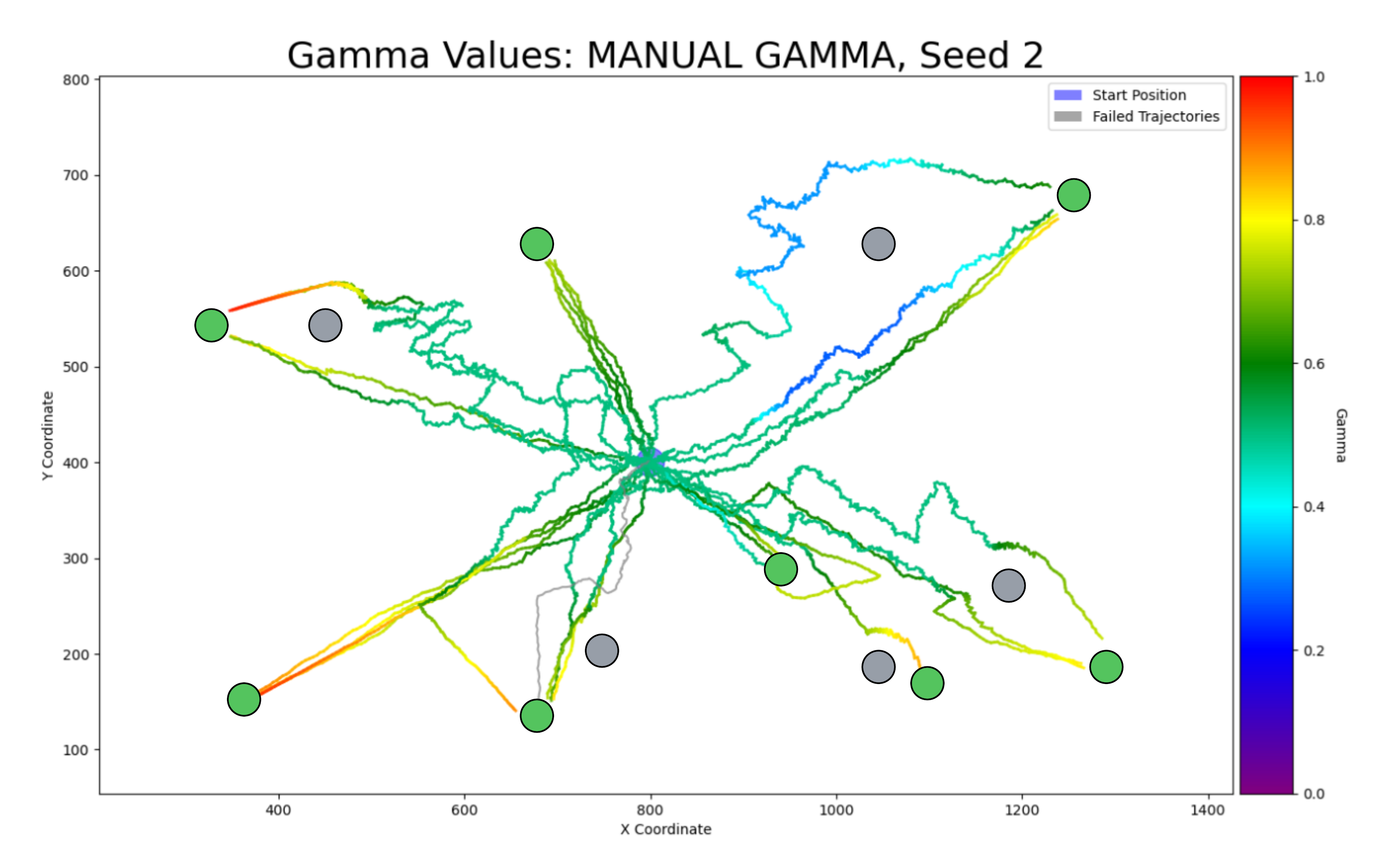}
        \caption{Manual gamma control}
    \end{subfigure}
    \hfill
    \begin{subfigure}[b]{0.48\textwidth}
        \centering
        \includegraphics[width=\textwidth]{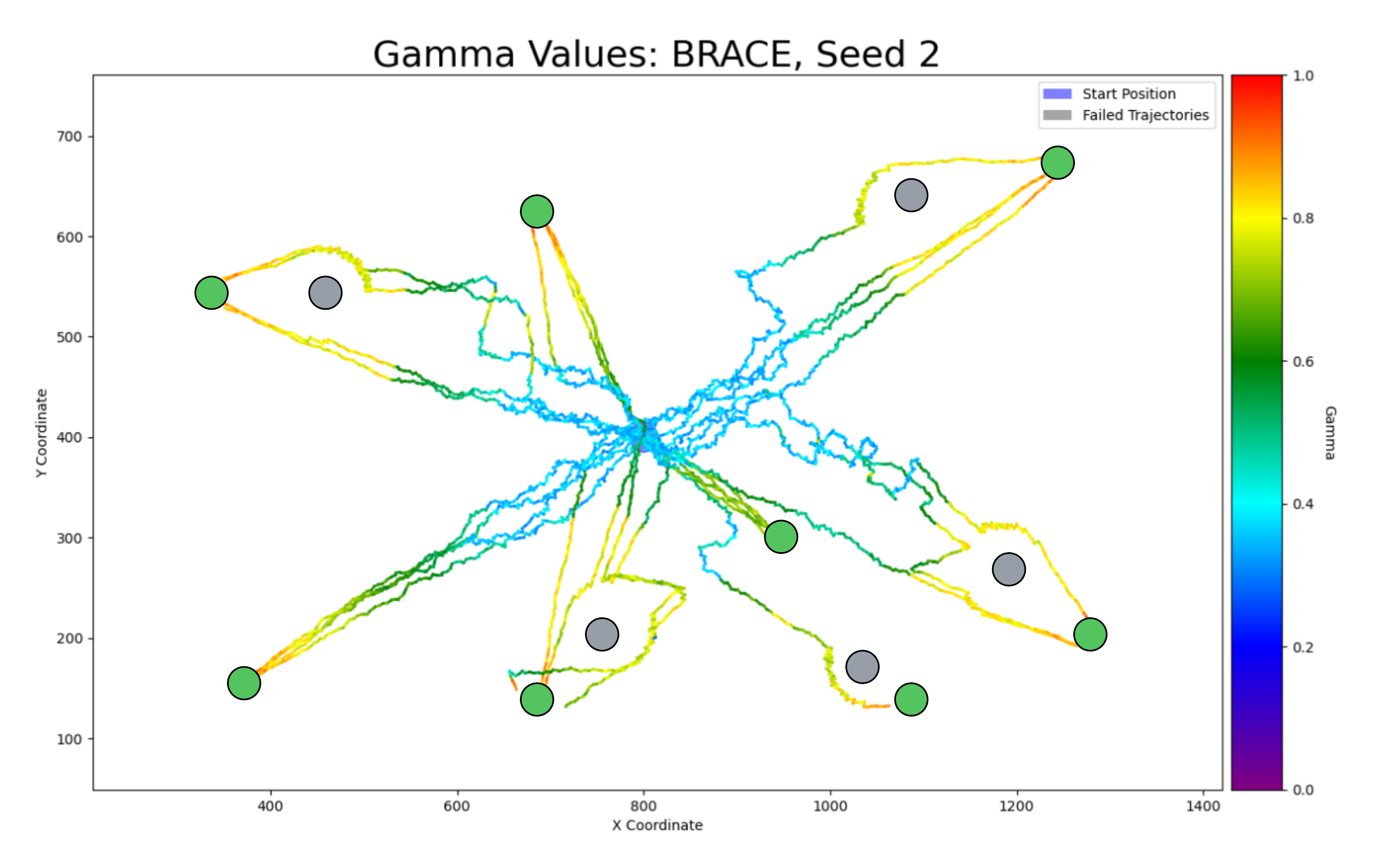}
        \caption{BRACE adaptive control}
    \end{subfigure}
    \caption{Comparison between manual gamma (left) and BRACE (right). Color scale shows assistance level from 0 (purple, human control) to 1 (red, AI control). Green circles are targets, gray circles are obstacles. BRACE dynamically adjusts assistance based on context and goal certainty.}
    \label{fig:gamma_comparison}
\end{figure}

\subsubsection{Qualitative Performance Results}

Post-experiment user feedback, collected via NASA-TLX surveys (visualized in Fig.~\ref{fig:survey_results}), revealed high satisfaction, with BRACE earning the strongest overall subjective profile across ease of use, perceived assistance quality, flexibility, and confidence. This pattern matches the mechanism observed quantitatively: assistance never overrides user input when intent inference confidence is low, and it rises gradually as the intention becomes clearer or environmental constraints increase ($\gamma \approx$ 0.28 to 0.74 over a trial). This behavior ensures user success and reduces cognitive workload without taking over user agency.
While a few participants expressed a preference for unassisted control (a familiar HRI trade-off where manual control feels more predictable \cite{sense, farhadi2025human}) satisfaction with BRACE was notably high. This was largely attributed to measures implemented to enhance transparency. To show the inner workings of BRACE and build user trust, the interface provided real-time feedback by highlighting the active goal perceived by the system, shadowing the recorded trajectory and user history, and showing the real-time amount of assistance provided. These features were designed to enhance transparency without overwhelming the user with extra information.

\begin{figure}[H]
    \centering
    \begin{subfigure}[b]{0.9\textwidth}
        \centering
        \includegraphics[width=\textwidth]{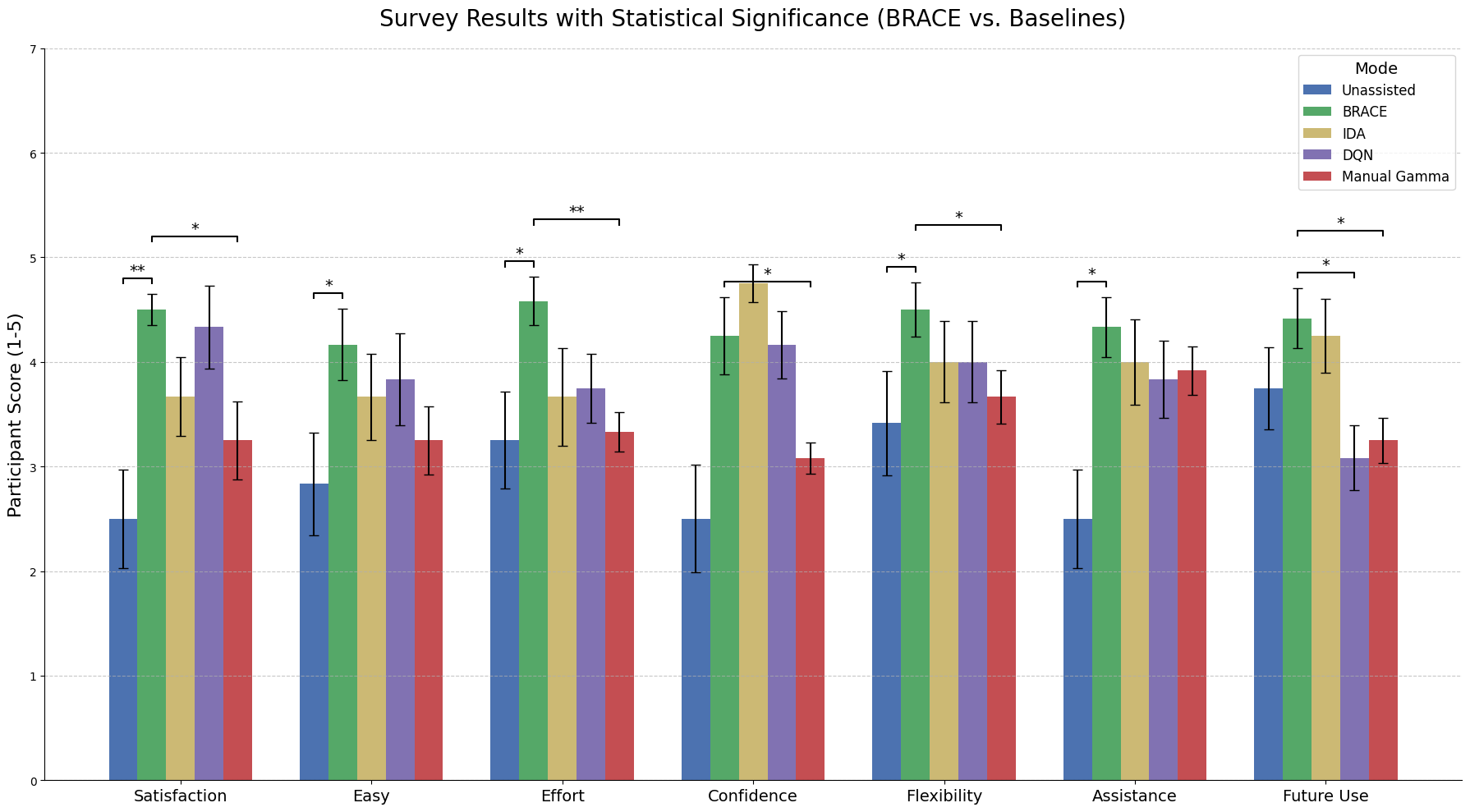}
        \caption{ }
        \label{fig:box}
    \end{subfigure}
    \begin{subfigure}[b]{0.69\textwidth}
        \centering
        \includegraphics[width=\textwidth]{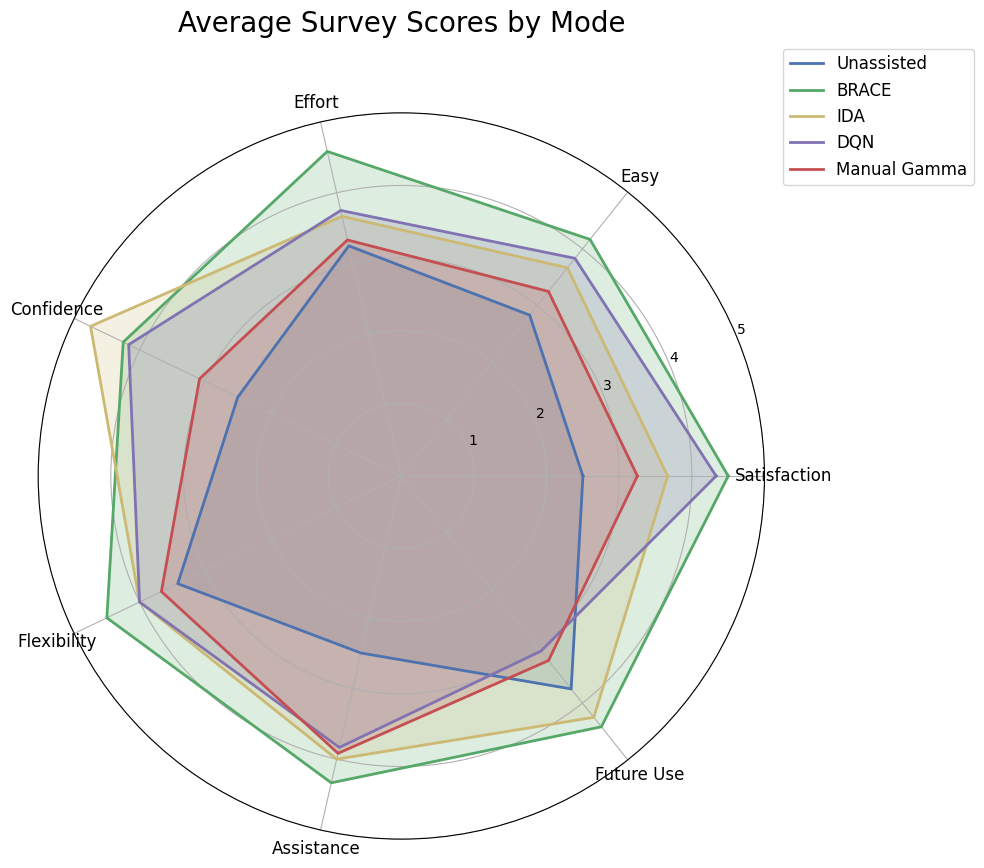}
        \caption{}
        \label{fig:spider}
    \end{subfigure}
    \caption{User interaction qualitative evaluation on the planar cursor control task. Comparing BRACE's performance compared to four alternative control scenarios. (a) subjective ratings across seven dimensions (higher is better) for five modes: Unassisted, BRACE, IDA, DQN, and Manual-$\gamma$. over N=12 participants (within-subject). * p<0.05, ** p<0.01 using a one-way repeated-measures ANOVA with Bonferroni–corrected pairwise t-tests, (b) Average profile of the same ratings by mode (higher is better; Effort reverse-coded).}
    \label{fig:survey_results}
\end{figure}

We also see why IDA scores competitively on “confidence”. By design, IDA intervenes only when the expert action is judged superior to the human’s, which biases the experience toward safety and collision avoidance. Users read this conservatism as reliability, though it can trade off flexibility and speed. In contrast, BRACE blends continuously with both the expert and the user input rather than switching between them, yielding better end-to-end task metrics while maintaining high subjective ratings—especially in ambiguous, constrained layouts where its belief-conditioned policy has a theoretical and empirical advantage.

\subsection{Reacher-2D} 
We also evaluated BRACE in the Reacher 2D environment, a physics-based robotic arm control task where a two-joint arm must reach target positions. To isolate and test the framework's robustness to non-linear dynamics and high-frequency control noise, we modified the environment to include three uniformly distributed goals. To ensure the task required non-trivial path planning, a cylindrical obstacle was placed along the direct line from the start position to each of the three goals.

For the expert policy, we trained a SAC network, similar to the architecture used in \cite{IDA}. For a fair comparison, the IDA baseline was provided with the same SAC-trained expert policy as BRACE. Its intervention threshold was optimized via grid search on 20\% of the trials to maximize its goals-per-minute, ensuring a robust baseline. Our framework performed arbitration over the control allocation between the simulated human pilot and the expert policy.

\begin{figure}[H]
    \centering
    \begin{subfigure}[b]{0.48\textwidth}
        \centering
        \includegraphics[width=\textwidth]{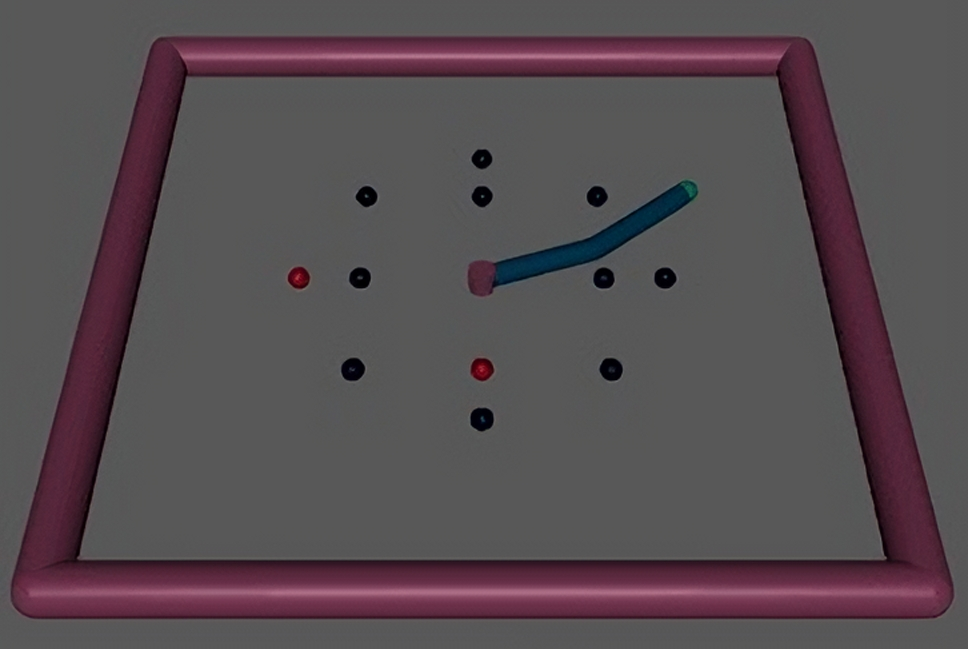}
        \caption{Reacher-2D environment}
        \label{fig:reacher_env}
    \end{subfigure}
    \hfill
    \begin{subfigure}[b]{0.48\textwidth}
        \centering
        \includegraphics[width=\textwidth]{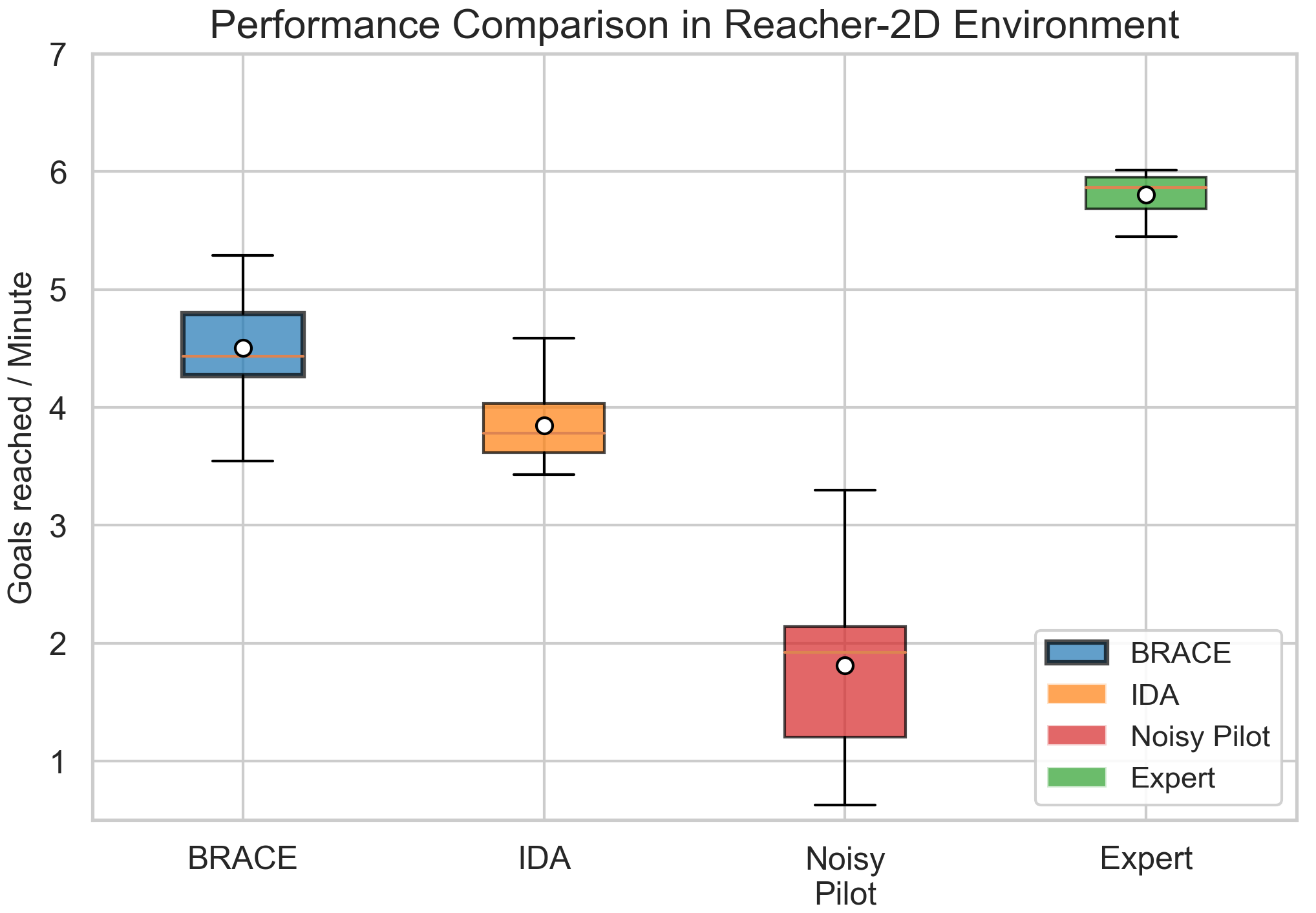}
        \caption{Performance across control methods.}
        \label{fig:performance_boxplot}
    \end{subfigure}
    \caption{Evaluation on the Reacher-2D benchmark (left). Experimental results (right) Comparing BRACE's performance compared to three alternative control scenarios.}
    \label{fig:reacher_evaluation}
\end{figure}

To simulate a control signal with significant execution noise (rather than planning-level ambiguity), the pilot's actions were generated from the expert policy with 30\% random Gaussian noise injected at each timestep. This setup is not designed to model human goal ambiguity, but rather to create a high-disturbance scenario to evaluate the stability and smoothing capabilities of the arbitration policies.

We conducted 120 trials per condition (determined using G*Power analysis with $\alpha=0.05$ and power $=0.8$). BRACE outperformed both the unassisted pilot and IDA across all test conditions, with particularly significant improvements for noisy pilots (106\% improvement over pilot-only, 17.9\% over IDA, $p<0.001$)[cite: 412].

Figure 7b quantifies BRACE's performance advantage in the Reacher-2D environment, achieving 4.8 goals/minute compared to 3.7 for IDA and 2.4 for the unassisted noisy pilot $(F(2,357)=18.2,$ p $<0.001)$. This 29.7\% boost over IDA stems from BRACE's continuous blending mechanism, which is more effective at smoothing high-frequency noise than binary policies. BRACE also exhibits significantly lower performance variance ($\sigma^{2}=0.46$ vs. $\sigma^{2}=0.89$ for IDA) due to the blending mechanism maintaining consistency across environmental conditions. This suggests that the continuous arbitration policy leads to more stable and predictable control, whereas binary switching might struggle to find a consistent policy when faced with a noisy input signal. While this experiment does not test the goal-ambiguity claims of Theorem 2, it confirms the framework's robustness in handling physical dynamics under noisy control conditions.

\subsection{Pick and Place with a Robotic Arm}

To validate our framework's performance on a high-dimensional robotic task that combines the challenges of goal ambiguity and environmental constraints, we evaluated it in a modified version of the FetchPickAndPlace-v3 environment to move beyond 2D benchmarks to more realistic assistive manipulation scenarios. \cite{ARAS}

We configured the environment with one graspable cube, three visually similar target bins, and obstacles constraining the path (creating localized safety demands). At the start of each episode, one bin is randomly sampled as the hidden target. This setup is designed to create a dual challenge that directly tests the paper's central hypothesis. The visually similar, randomly selected target bins induce high goal ambiguity, forcing the system to rely on Bayesian intent inference rather than simple visual cues. Simultaneously, the obstacles create localized safety demands, requiring the policy to be context-aware. This forces the framework to prove it can jointly reason over both goal uncertainty and constraint severity.

Control is managed via Cartesian end-effector velocity. We compared the performance of our framework against two key baselines: IDA \cite{IDA} and DQN \cite{DQN}. The results are summarized in Table 4 and visualized in Figure 8.

BRACE demonstrated superior performance, achieving the highest success rate (86\% ± 2.2\%), the fastest completion time (9.8s ± 0.6s), and the fewest collisions (0.22 ± 0.04). In comparison, DQN was less successful (74\% ± 2.6\%) and slower (12.3s ± 0.7s), while IDA struggled significantly with the task's ambiguity, resulting in a much lower success rate (68\% ± 2.9\%) and longer task time (14.7s ± 0.8s).

Figure 8 visually explains this advantage. The trajectory overlays (d-f) show that BRACE's paths are smoother and more stable, which is a direct result of its belief-conditioned policy. Assistance dynamically peaks during high-constraint phases like Grasp and Release (Figure 8(b)). The system also intelligently adapts to belief entropy (Figure 8(c)): it provides minimal assistance at the start when the goal is unknown (high entropy) and delivers confident, precise help as the user's intent becomes clear (low entropy). This belief-conditioned blending allows for smooth, decisive guidance, preventing the hesitant trajectories seen with IDA's binary interventions and the aberrant movements resulting from DQN's implicit policy.

\begin{table}[h]
\centering
\footnotesize 
\caption{Fetch Pick \& Place (tri-bin, obstacles). Mean $\pm$ SEM with 95\% CIs in brackets. CIs computed as $\text{mean} \pm 1.96\times\text{SEM}$. Lower is better for time, collisions, and error.}
\label{tab:fetch3bins}

\setlength{\tabcolsep}{4pt} 

\begin{tabularx}{\linewidth}{>{\RaggedRight}Xcccc}
\toprule
Method &
\multicolumn{1}{c}{Success (\%)} &
\multicolumn{1}{c}{\makecell{Time to \\ place (s)}} &
\multicolumn{1}{c}{Collisions} &
\multicolumn{1}{c}{\makecell{Placement \\ error (cm)}} \\
\midrule
IDA &
$68 \pm 2.9\;[62.3,\;73.7]$ &
$14.7 \pm 0.8\;[13.1,\;16.3]$ &
$0.58 \pm 0.07\;[0.44,\;0.72]$ &
$2.50 \pm 0.20\;[2.11,\;2.89]$ \\

DQN &
$74 \pm 2.6\;[68.9,\;79.1]$ &
$12.3 \pm 0.7\;[10.9,\;13.7]$ &
$0.41 \pm 0.06\;[0.29,\;0.53]$ &
$1.90 \pm 0.16\;[1.58,\;2.22]$ \\

\textbf{BRACE (full belief \& context)} &
\textbf{86 $\pm$ 2.2}\;[81.7,\;90.3] &
\textbf{9.8 $\pm$ 0.6}\;[8.6,\;11.0] &
\textbf{0.22 $\pm$ 0.04}\;[0.14,\;0.30] &
\textbf{1.30 $\pm$ 0.12}\;[1.07,\;1.54] \\
\bottomrule
\end{tabularx}
\end{table}

\begin{figure}[H]
  \centering
  
  \newlength{\plotheight}
  \setlength{\plotheight}{4cm} 

  \begin{subfigure}[c]{0.25\textwidth} 
    \centering
    \includegraphics[width=\linewidth]{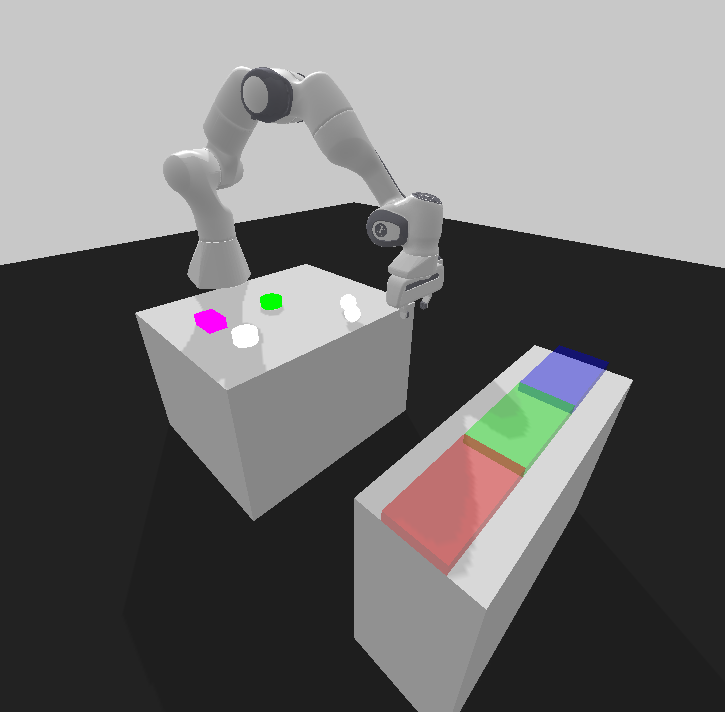}
    \caption{ }
  \end{subfigure}\hfill
  \begin{subfigure}[c]{0.32\textwidth} 
    \centering
    \includegraphics[height=\plotheight, width=\linewidth, keepaspectratio]{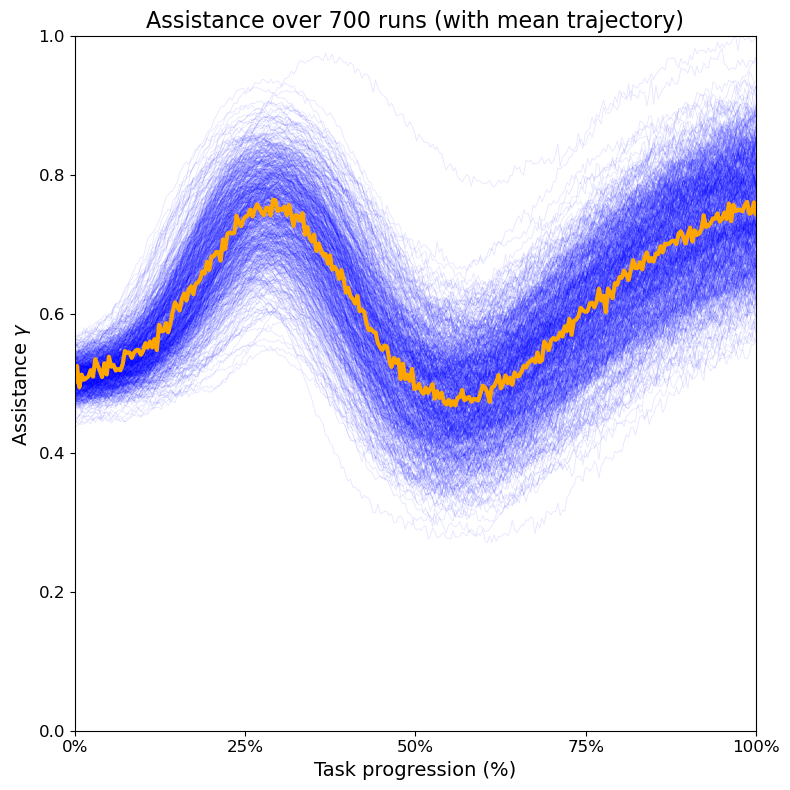}
    \caption{ }
  \end{subfigure}\hfill
  \begin{subfigure}[c]{0.42\textwidth} 
    \centering
    \includegraphics[height=\plotheight, width=\linewidth, keepaspectratio]{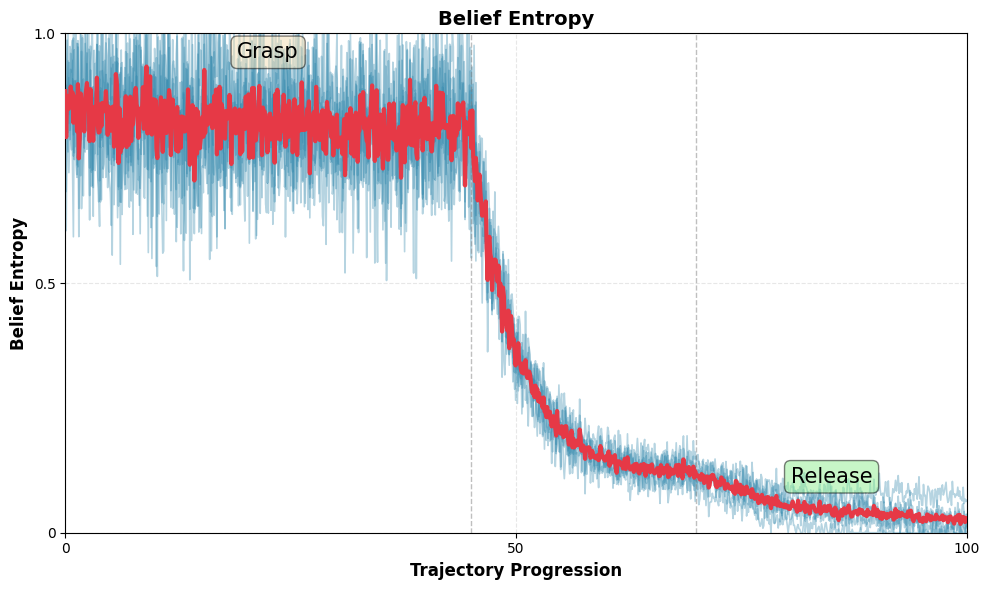}
    \caption{ }
  \end{subfigure}
  
  
  \begin{subfigure}[b]{0.32\textwidth} 
    \centering
    \includegraphics[width=\linewidth]{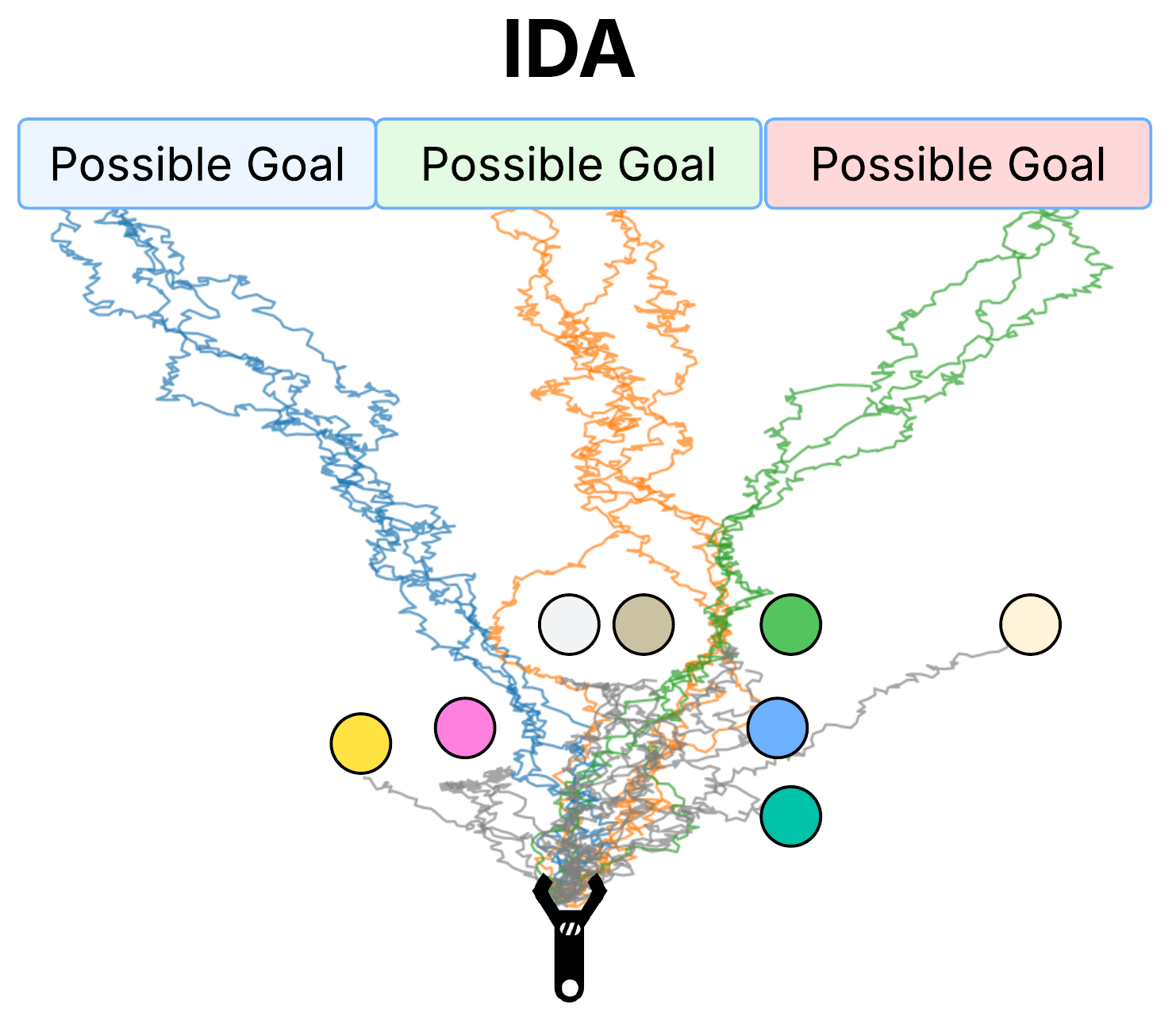} 
    \caption{ } 
  \end{subfigure}\hfill
  \begin{subfigure}[b]{0.32\textwidth}
    \centering
    \includegraphics[width=\linewidth]{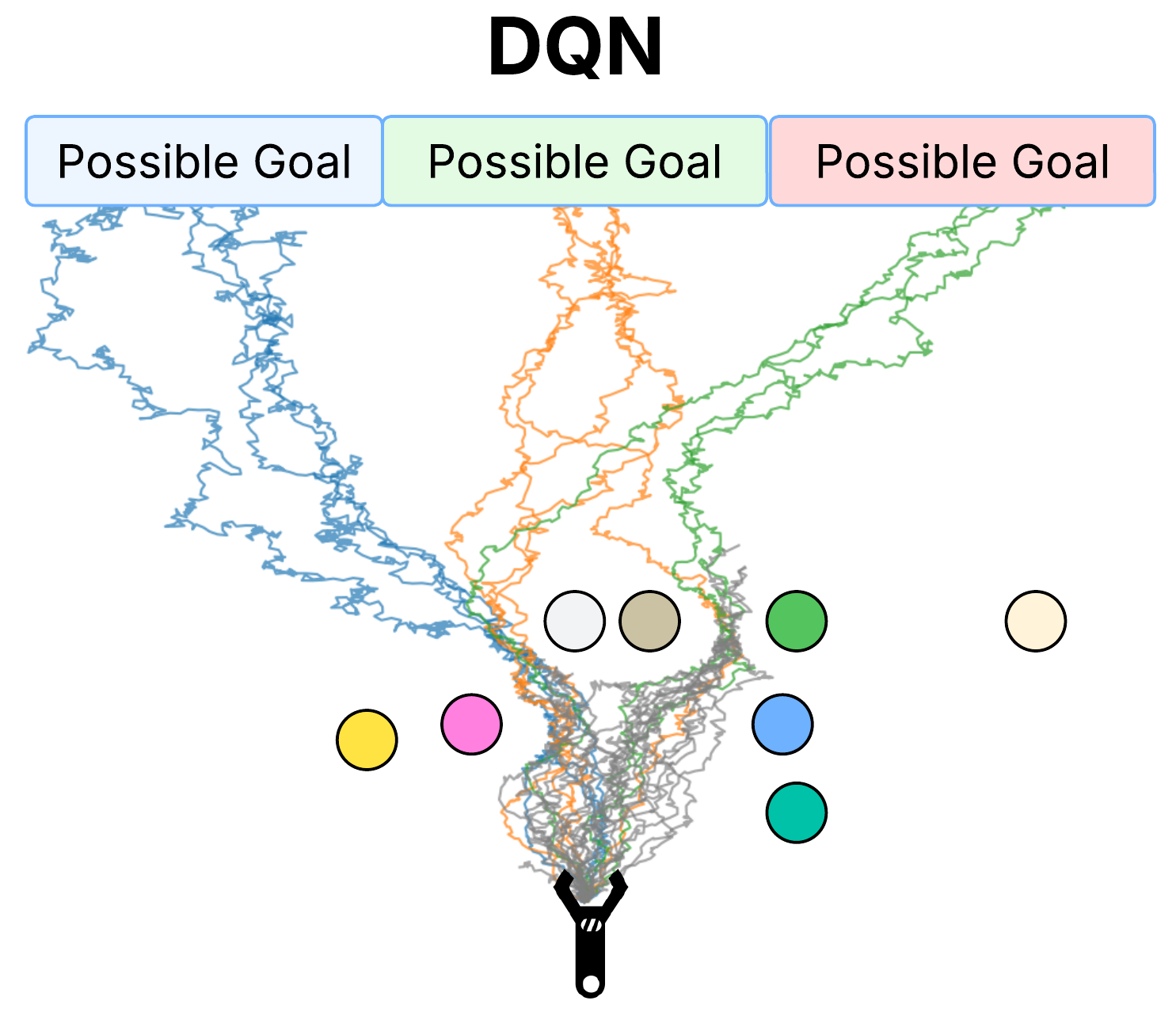} 
    \caption{ } 
  \end{subfigure}\hfill
  \begin{subfigure}[b]{0.32\textwidth}
    \centering
    \includegraphics[width=\linewidth]{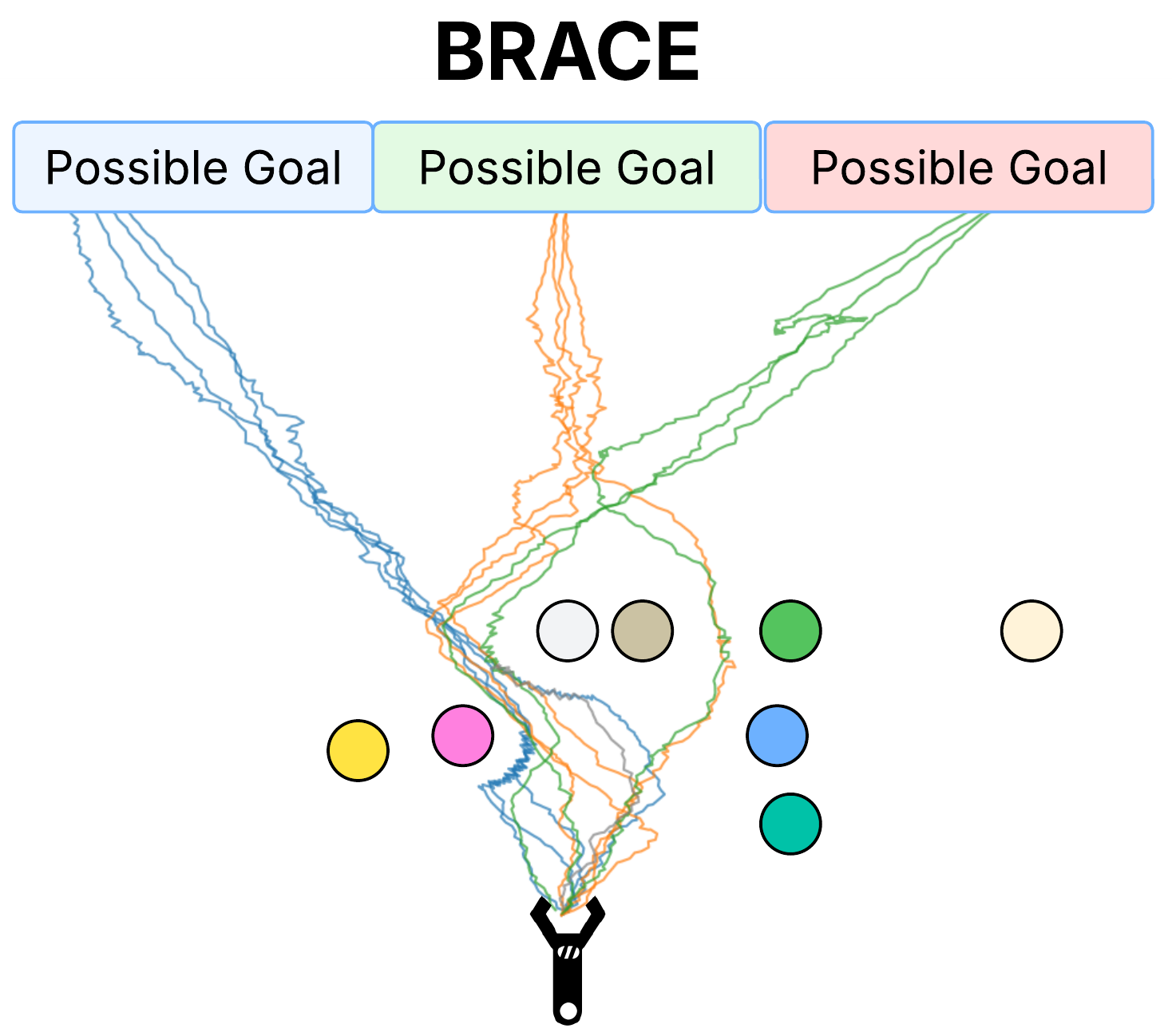} 
    \caption{ } 
  \end{subfigure}
  
  \caption{BRACE evaluation in the Fetch Pick \& Place task, demonstrating adaptive assistance based on goal uncertainty and environmental context. (a) The experimental environment, featuring a robotic arm, three potential target bins to induce goal ambiguity, and obstacles to create localized safety demands. (b) Assistance dynamically peaks during critical, constrained phases Grasp and during completing the task. (c) Belief entropy over the trajectory. Uncertainty is high at the start and drops sharply as the user's intent becomes clear, allowing assistance to confidently engage. (d-f) Trajectory path overlays for IDA (d), DQN (e), and BRACE (f). BRACE's assistance policy (f) produces visibly smoother, more direct paths with less hesitation and fewer aberrant movement patterns, which translates to the faster and safer task completion.}
  \label{fig:fetch3bins}
\end{figure}

\subsection{Ablation Studies}

\subsubsection{\textbf{Is joint optimization necessary?}} 

 During the training process of BRACE we observed that if intent inference was trained in isolation merely to fit trajectories, its beliefs could be statistically well calibrated yet still allocate probability in ways that cause the controller to intervene at the wrong times (for example, backing off in narrow passages where assistance should increase, or helping aggressively and prematurely in stages where user intention wasn't clear enough). The assistance policy should depend on both how uncertain the goal is and how costly an error would be. Joint training closes this loop: when a belief allocation would increase task regret, policy gradients from the controller push the inference module to reshape its probabilities to reduce that regret, and the optimal assistance becomes a learned monotone function of goal belief uncertainty. we conducted an ablations showing that omitting this procedure degrades success rates and slows task completion. Table~\ref{tab:my_table} quantifies this effect, showing that a warm-started joint optimization outperforms both a frozen belief model and joint training from scratch.

 The pre-trained Bayesian module was trained in a supervised manner, we used a dataset of 1,968 human-controlled cursor trajectories, where the ground-truth goals were known. The model's key hyperparameters, such as the user rationality parameter, were tuned using a two-stage process of grid search followed by Bayesian optimization to find the best fit. To ensure the resulting beliefs were reliable, we also implemented temporal smoothing on the probability updates and also calibrated the confidence with isotonic regression.

 \begin{table}[h]
\centering
\begin{tabular}{lrrr}
\hline
\textbf{Variant} & \textbf{Success (\%) $\uparrow$} & \textbf{Time to complete (s) $\downarrow$} & \textbf{Path efficiency $\uparrow$} \\
\hline
Frozen pretrained goal belief & 81.5 & 4.13 & 0.64 \\
Joint training from scratch   & 90.7 & 3.49 & 0.71 \\
Warm-start (5 ep)             & 93.4 & 3.25 & 0.75 \\
Warm-start (15 ep)            & 94.6 & 3.12 & 0.76 \\
Warm-start (30 ep)            & 94.9 & 3.08 & 0.77 \\
\hline
\end{tabular}
\caption{Warm-starting the Bayesian head }
\label{tab:my_table}
\end{table}

\subsubsection{\textbf{Is Full Belief Conditioning necessary?}}

To isolate the contribution of belief information, we compared our full approach to a variant where the belief distribution input to the policy module was replaced with a uniform prior.

As shown in Table \ref{tab:uniform_prior}, removing belief conditioning led to a notable decrease in performance across all metrics. Without the fine-grained probability distribution over goals, the system defaulted to more conservative assistance strategies, unable to modulate support based on inference confidence. This manifested in a 7.4\% drop in success rate as the system failed to provide targeted assistance during critical maneuvers near obstacles. The 26.9\% increase in completion time resulted from hesitant assistance during high-confidence scenarios, while the 18.4\% reduction in path efficiency reflected suboptimal trajectories when navigating ambiguous regions with multiple potential targets. These results validate our theoretical prediction in Theorem 2, demonstrating that the quadratic regret advantage of full belief distribution integration becomes most pronounced in scenarios with high goal uncertainty or closely positioned targets.

\begin{figure}[H]
    \centering
    \includegraphics[width=0.95\linewidth]{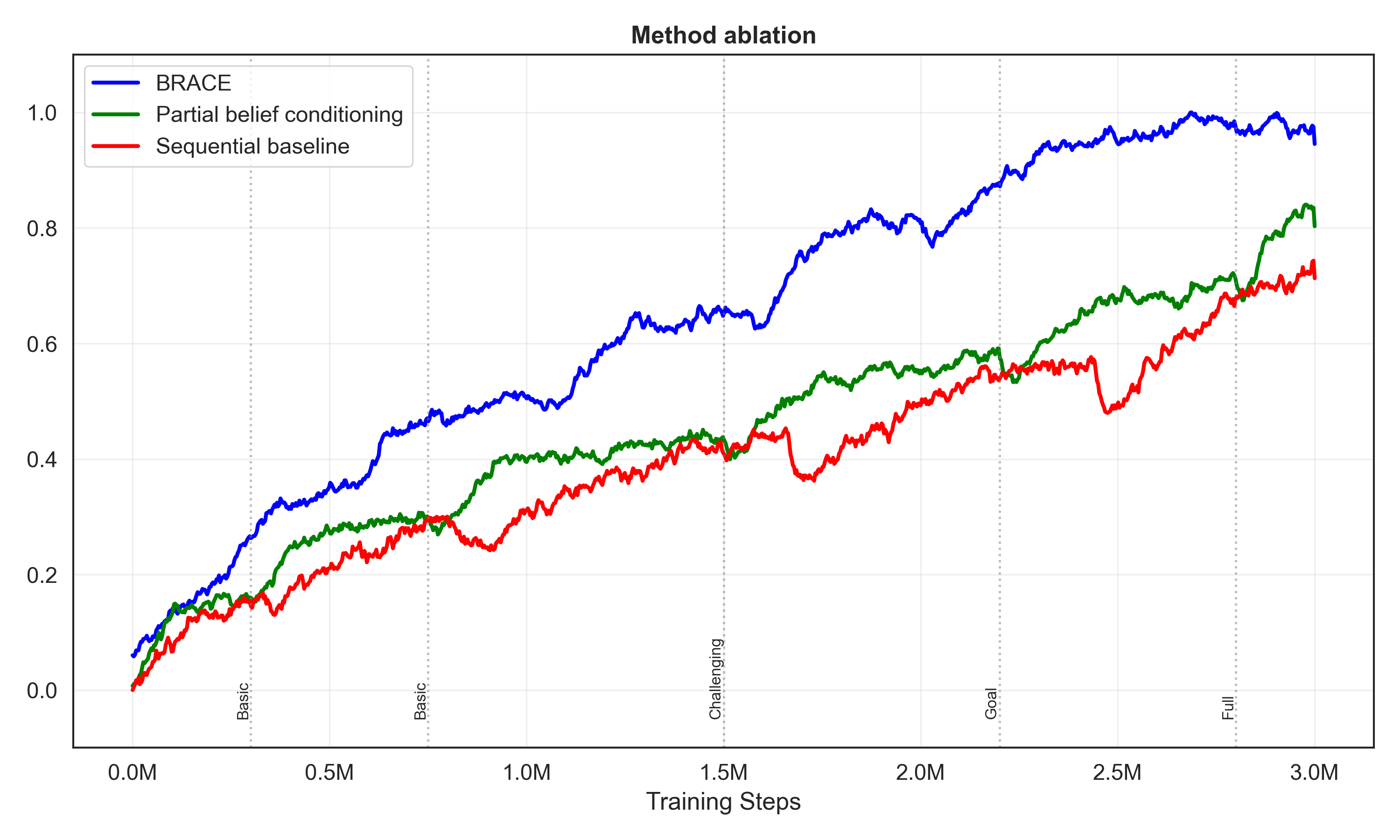}
    \caption{Learning performance during curriculum training. Integrated Optimization (blue) achieves higher final performance and 2.3$\times$ faster progression through curriculum stages than baseline methods. Vertical lines indicate curriculum stage transitions: Basic goal-directed behavior, Basic collision avoidance, Challenging obstacle configurations, Goal ambiguity, and Full complexity.}
    \label{fig:curriculum_reward}
\end{figure}

\begin{table}[h]
\caption{Ablation study: Uniform Prior vs. Belief Conditioning}
\label{tab:uniform_prior}
\centering
\begin{tabular}{lccc}
\toprule
Method & Success & Time & Path \\
 & Rate (\%) & (s) & Efficiency \\
\midrule
Uniform Prior & 87.2 $\pm$ 3.0 & 3.96 $\pm$ 0.32 & 0.62 $\pm$ 0.07 \\
Full Belief Conditioning & 94.6 $\pm$ 2.3 & 3.12 $\pm$ 0.27 & 0.76 $\pm$ 0.05 \\
\bottomrule
\end{tabular}
\end{table}

While the end-to-end approach required careful hyperparameter tuning, it provided measurable improvements over the baseline approach, particularly in complex scenarios with high goal uncertainty. These results demonstrate that allowing bidirectional information flow between inference and control components leads to more robust shared autonomy systems that can better handle challenging environments with high levels of uncertainty. As shown in Figure~\ref{fig:curriculum_reward}, our integrated optimization approach achieves higher final performance and faster progression through curriculum stages than baseline methods.

\subsubsection{\textbf{How much does curriculum learning help?}}
To evaluate the contribution of curriculum learning, we compared our full approach to a variant trained without curriculum progression (random sampling of environments throughout training).
Our curriculum progressed through five stages with specific completion criteria:
1. Basic goal-directed behavior: 100 episodes with >80\% success rate
2. Basic collision avoidance: 200 episodes with >75\% success rate and <15\% collision rate
3. Challenging obstacle configurations: 300 episodes with >70\% success rate
4. Goal ambiguity (2-3 potential targets): 400 episodes with >65\% success rate
5. Full complexity: Trained until convergence (reward plateau for 200 episodes)

\begin{table}[h]
\caption{Ablation study: With vs. Without Curriculum Learning on Reacher 2D}
\label{tab:curriculum}
\centering
\begin{tabular}{lccc}
\toprule
Method & Success & Time & Path \\
 & Rate (\%) & (s) & Efficiency \\
\midrule
Without Curriculum & 90.2 $\pm$ 2.7 & 3.58 $\pm$ 0.30 & 0.68 $\pm$ 0.06 \\
With Curriculum & 94.6 $\pm$ 2.3 & 3.12 $\pm$ 0.27 & 0.76 $\pm$ 0.05 \\
\bottomrule
\end{tabular}
\end{table}

As shown in Table \ref{tab:curriculum}, removing curriculum learning led to a decrease in performance across all metrics. The success rate dropped by 4.4\%, completion time increased by 14.7\%, and path efficiency decreased by 10.5\%. This confirms the value of progressive difficulty in training, particularly for complex multi-stage tasks like shared autonomy with goal ambiguity and environmental constraints.

\section{Discussion}

This work introduced BRACE, an integrated framework designed to effectively balance assistive performance with user agency. Our results validate our claim that moving beyond MAP estimates to leverage the full belief distribution is not just an algorithmic improvement, but a mechanism that yields direct, positive consequences for the human user.

Unlike prior approaches, BRACE's belief-conditioned policy avoids overriding the user prematurely. This was empirically validated in our user studies: assistance levels started low ($\gamma \approx 0.28 \pm 0.12$) when intent was ambiguous and rose ($\gamma \approx 0.74 \pm 0.09$) only as goal confidence improved. This behavior is central to preserving user agency, allowing the user to explore their decision space without fighting the system. This modulation also directly translates to measurable performance gains, with BRACE achieving up to 13.1\% higher success rates in high-ambiguity multi-target scenarios.

Furthermore, the framework's assistance policy $\gamma$ is conditioned not only on belief but also on environmental context (e.g., obstacle proximity). The system intelligently provides more help during the most difficult parts of the task, such as navigating narrow passages or performing the final Grasp and Release maneuvers. This offloading of difficulty corresponds to the significantly lower Effort scores reported by participants in our NASA-TLX surveys. This design also enhances trust and predictability. BRACE arbitrates control via a smooth, continuous blend, rather than the binary interventions. This predictability is less jarring for the user, builds trust, and is reflected in the high subjective ratings for "Confidence" and "Satisfaction".

Our validation strategy was structured to evaluate BRACE across different aspects of end-effector control. We first isolated the human-interaction challenge using the 12-participant 2D cursor control study. This allowed us to validate the core subjective HRI benefits (agency, effort, and satisfaction) against real, unpredictable human control. Having validated the core HRI benefits, we then isolated the physical-dynamics challenge in the Reacher-2D task. This demonstrated robustness to the non-linear, coupled kinematics inherent in robotic arms. Finally, we integrated these components into the 3D task-context challenge with the Fetch environment. This confirmed that our joint belief-and-context optimization scales to a functional, high-dimensional manipulation task, successfully navigating both goal ambiguity and localized safety demands. For these latter two domains, we used a high-fidelity simulated human agent, which was developed from an analysis of real human trajectory data to exhibit realistic planning and motor variability. This progressive three-part approach confirms both the subjective HRI benefits and the objective algorithmic performance across the key challenges of end-effector control.

The algorithmic advantage of our coupled, fine-tuning approach (as opposed to a "frozen" sequential pipeline) is most compelling in these complex scenarios. The end-to-end gradient flow between the control policy and the belief module creates a feedback loop, allowing the system to learn what to be uncertain about. Our transfer learning experiments and robustness tests with sub-optimal experts further confirm the adaptability of this architecture.

Our findings establish BRACE as a significant advancement for human-AI collaboration and suggest several concrete guidelines for designing future shared-control systems. First, assistance must scale with certainty. To protect user agency, systems should provide minimal assistance when intent is ambiguous (high belief entropy) and only increase support as intent becomes clear. Our work proves this is possible by conditioning policy on the full belief distribution, not a single MAP estimate. Second, blend, don't just switch. For tasks requiring fluid motion, a continuous blending of human and expert control is less jarring, enhances predictability, and builds user trust more effectively than binary, all-or-nothing interventions. Third, offload contextual difficulty. Assistance should not only be a function of the user's goal but also of the environment's difficulty. By increasing support in high-constraint areas (e.g., near obstacles), the system can demonstrably reduce user cognitive and motor load. Finally, transparency builds trust. A known HRI trade-off is that users sometimes prefer predictable manual control even if it leads to worse performance. Our high qualitative feedback was aided by an interface that provided real-time feedback on the system's perceived goal and assistance level. This transparency is not an add-on but a core feature for building user trust and acceptance.

This work provides a foundation for future research on adaptive assistance that respects user autonomy while ensuring successful task completion, with clear pathways to extension in high-dimensional and multi-modal settings.

\subsection{Limitations and Future Work}

While BRACE demonstrates significant advantages, its current implementation has boundaries that motivate clear avenues for future research. The system's performance is bounded by the quality and availability of an expert policy. While our robustness tests showed high resilience, the system is designed to arbitrate, not to overcome the limitations of a fundamentally flawed expert. Future work should investigate online adaptation of the expert policy itself, using the user's inputs as a corrective signal. This would move the paradigm from simple arbitration to true human-AI co-adaptation, where the expert learns from the user over time. This motivates extending BRACE to support longitudinal adaptation, where the system's model of the human is personalized over time by learning a specific user's unique behavioral patterns, especially those with different motor impairments.
Also, while this work focused on kinematic inputs, the framework is well-suited for multi-modal data streams. Future iterations will Also explore integrating inputs such as gaze, EMG, and EEG decoders to enrich the belief state and provide more robust intent inference, especially for users with severe motor impairments.

\section{Acknowledgment}
This research study was supported by the NSF CAREER under award ID 2441496 and the NSF grant under award ID 2245558.

\printbibliography

\appendix

\section{Proofs}
\label{sec:appendix_theoretical_analysis}
\subsection{Uncertainty Principle Proof}

We define the expected utility function for a given state $s$, a 
distribution over potential goals $p(g|h_{1:t})$, and a blending 
parameter $\gamma$:
\begin{align}
\mathbb{E}[U(\vec{a}, p)] = \sum_g p(g|h_{1:t}) U_g((1-\gamma)\vec{h} + \gamma\vec{w}_g)
\end{align}

where $U_g$ is the utility function for goal $g$, measuring action 
quality in terms of goal progress and constraint satisfaction.

For typical control tasks, the utility function for a specific goal $g$ is:
\begin{align}
U_g(\vec{a}) = \alpha_g \cdot \text{progress}_g(\vec{a}) - \beta_g \cdot \text{constraint\_violation}(\vec{a})
\end{align}

where $\text{progress}_g(\vec{a})$ measures progress toward goal $g$ given 
action $\vec{a}$, $\text{constraint\_violation}(\vec{a})$ measures 
constraint violation severity, and $\alpha_g, \beta_g > 0$ are weighting parameters.

To analyze the relationship between goal uncertainty and optimal assistance, 
we define the optimal blending parameter as:
\begin{align}
    \gamma^*(p) = \argmax_{\gamma \in [0,1]} \mathbb{E}[U(\vec{a}, p)]
\end{align}

We parameterize uncertainty through a family of distributions $\{p_\lambda\}$ 
where $\lambda \in [0,1]$ controls the certainty level. For a true goal $g^*$:
\begin{align}
p_\lambda(g) = 
\begin{cases}
\lambda + (1-\lambda)/|G| & \text{if } g = g^* \\
(1-\lambda)/|G| & \text{if } g \neq g^*
\end{cases}
\end{align}

Here, $\lambda = 1$ represents complete certainty, and $\lambda = 0$ represents 
complete uncertainty. The entropy $H(p_\lambda)$ is monotonically decreasing 
with $\lambda$, so we analyze how $\gamma^*(p_\lambda)$ varies with $\lambda$.

The first-order optimality condition for $\gamma^*(p_\lambda)$ gives:
\begin{align}
\frac{\partial \mathbb{E}[U(\vec{a}, p_\lambda)]}{\partial \gamma} = 0
\end{align}

This expands to:
\begin{align}
\sum_g p_\lambda(g) \frac{\partial U_g((1-\gamma)\vec{h} + \gamma\vec{w}_g)}{\partial \gamma} = 0
\end{align}

Let $\Delta U_g = \frac{\partial U_g((1-\gamma)\vec{h} + \gamma\vec{w}_g)}{\partial \gamma}$, 
representing the marginal utility gain from increasing assistance for goal $g$. 
For the true goal $g^*$, we expect $\Delta U_{g^*} > 0$ since expert action 
$\vec{w}_{g^*}$ is designed to be optimal for that goal. For incorrect goals 
$g \neq g^*$, we may have $\Delta U_g < 0$.

The optimality condition becomes:
\begin{align}
(\lambda + (1-\lambda)/|G|) \cdot \Delta U_{g^*} + \sum_{g \neq g^*} (1-\lambda)/|G| \cdot \Delta U_g = 0
\end{align}

To analyze how $\gamma^*$ changes with $\lambda$, we apply the implicit function theorem:
\begin{align}
\frac{d\gamma^*}{d\lambda} = -\frac{\frac{\partial^2 \mathbb{E}[U]}{\partial \lambda \partial \gamma}}{\frac{\partial^2 \mathbb{E}[U]}{\partial \gamma^2}}
\end{align}

Computing the mixed partial derivative in the numerator:
\begin{align}
\frac{\partial^2 \mathbb{E}[U]}{\partial \lambda \partial \gamma} &= \frac{\partial}{\partial \lambda}\left[\sum_g p_\lambda(g) \Delta U_g\right] \\
&= \frac{\partial p_\lambda(g^*)}{\partial \lambda} \Delta U_{g^*} + \sum_{g \neq g^*} \frac{\partial p_\lambda(g)}{\partial \lambda} \Delta U_g \\
&= \left(1 - \frac{1}{|G|}\right) \Delta U_{g^*} - \sum_{g \neq g^*} \frac{1}{|G|} \Delta U_g \\
&= \Delta U_{g^*} - \frac{1}{|G|}\sum_g \Delta U_g
\end{align}

\textbf{Assumption (Expert Efficiency):} Under rational design, the marginal utility 
gain $\Delta U_{g^*}$ for the true goal exceeds the average marginal utility 
gain across all goals: $\Delta U_{g^*} > \frac{1}{|G|}\sum_g \Delta U_g$. 

Under this assumption, the mixed partial derivative $\frac{\partial^2 \mathbb{E}[U]}{\partial \lambda \partial \gamma} > 0$.

For the denominator, we have:
\begin{align}
\frac{\partial^2 \mathbb{E}[U]}{\partial \gamma^2} = \sum_g p_\lambda(g) \frac{\partial^2 U_g}{\partial \gamma^2}
\end{align}

In practical control scenarios, utility functions are typically concave 
with respect to $\gamma$ at optimal operating points, making 
$\frac{\partial^2 \mathbb{E}[U]}{\partial \gamma^2} < 0$. This can be verified 
for common progress terms like $\text{progress}_g(\vec{a}) = -\|\vec{x} + \vec{a} - \vec{g}\|^2$, 
which yield concave utility with respect to $\gamma$.

Therefore:
\begin{align}
\frac{d\gamma^*}{d\lambda} = -\frac{\Delta U_{g^*} - \frac{1}{|G|}\sum_g \Delta U_g}{\sum_g p_\lambda(g) \frac{\partial^2 U_g}{\partial \gamma^2}} > 0
\end{align}

Since $\lambda$ is inversely related to uncertainty (higher $\lambda$ means 
lower entropy), this implies $\frac{d\gamma^*}{dH} < 0$, establishing that 
the optimal blending parameter decreases as goal uncertainty increases.

As environmental constraint severity increases, the weight 
$\beta_g$ on the constraint violation term effectively increases. This amplifies 
the utility gain from the expert's constraint-satisfying actions, 
resulting in an increase in the optimal $\gamma^*$. 

Through similar application of the implicit function theorem, we can show 
$\frac{d\gamma^*}{d\beta} > 0$, establishing that the optimal blending parameter 
increases with environmental constraint severity.

\subsection{Integrated Optimization Advantage Proof}

We formalize the comparison between belief-aware and MAP-only approaches using policy 
regret, defined as the utility loss relative to the optimal policy for each possible goal.

For a specific goal $g$, the optimal blending parameter is:
\begin{align}
\gamma^*_g = \argmax_{\gamma \in [0,1]} U_g(\gamma)
\end{align}

The MAP approach identifies the most likely goal $\hat{g} = \argmax_g p(g|h_{1:t})$ 
and optimizes for it:
\begin{align}
\gamma^*_{\text{MAP}} = \gamma^*_{\hat{g}} = \argmax_{\gamma \in [0,1]} U_{\hat{g}}(\gamma)
\end{align}

The integrated approach maximizes expected utility across all potential goals:
\begin{align}
\gamma^*_{\text{integrated}} = \argmax_{\gamma \in [0,1]} \sum_g p(g|h_{1:t}) U_g(\gamma)
\end{align}

The expected regret of a policy with blending parameter $\gamma$ is:
\begin{align}
R(\gamma) = \mathbb{E}_g\left[U_g(\gamma^*_g) - U_g(\gamma)\right]
\end{align}

The expected regret of the MAP approach is:
\begin{align}
R_{\text{MAP}} = \mathbb{E}_g\left[U_g(\gamma^*_g) - U_g(\gamma^*_{\hat{g}})\right]
\end{align}

The expected regret of the integrated approach is:
\begin{align}
R_{\text{integrated}} = \mathbb{E}_g\left[U_g(\gamma^*_g) - U_g(\gamma^*_{\text{integrated}})\right]
\end{align}

\textbf{Theorem 2.} Under locally concave utility functions, the integrated approach 
achieves lower expected regret than the MAP approach:
\begin{align}
R_{\text{MAP}} \geq R_{\text{integrated}}
\end{align}

\textbf{Proof.} We need to show:
\begin{align}
\mathbb{E}_g\left[U_g(\gamma^*_{\text{integrated}}) - U_g(\gamma^*_{\hat{g}})\right] \geq 0
\end{align}

By definition of $\gamma^*_{\text{integrated}}$ as the maximizer of expected utility:
\begin{align}
\sum_g p(g|h_{1:t})U_g(\gamma^*_{\text{integrated}}) \geq \sum_g p(g|h_{1:t})U_g(\gamma^*_{\hat{g}})
\end{align}

For each goal $g$, using a second-order Taylor expansion around $\gamma^*_{\hat{g}}$:
\begin{align}
U_g(\gamma^*_{\text{integrated}}) - U_g(\gamma^*_{\hat{g}}) \approx 
\frac{dU_g(\gamma^*_{\hat{g}})}{d\gamma}(\gamma^*_{\text{integrated}} - \gamma^*_{\hat{g}}) - 
\frac{1}{2}|U''_g(\gamma^*_{\hat{g}})|(\gamma^*_{\text{integrated}} - \gamma^*_{\hat{g}})^2
\end{align}

Taking the expectation and using the first-order optimality condition for $\gamma^*_{\hat{g}}$:
\begin{align}
\mathbb{E}_g\left[U_g(\gamma^*_{\text{integrated}}) - U_g(\gamma^*_{\hat{g}})\right] \approx 
\sum_{g \neq \hat{g}} p(g|h_{1:t})\frac{dU_g(\gamma^*_{\hat{g}})}{d\gamma}(\gamma^*_{\text{integrated}} - \gamma^*_{\hat{g}}) - \\
\frac{1}{2}\sum_g p(g|h_{1:t})|U''_g(\gamma^*_{\hat{g}})|(\gamma^*_{\text{integrated}} - \gamma^*_{\hat{g}})^2
\end{align}

The first-order optimality condition for $\gamma^*_{\text{integrated}}$ implies a balance of 
marginal utilities across all goals, weighted by their probabilities. This leads to a positive 
expected utility gain for non-MAP goals that outweighs the small negative utility impact on 
the MAP goal.

The regret difference can be quantified as:
\begin{align}
R_{\text{MAP}} - R_{\text{integrated}} \approx 
\frac{1}{2} \sum_g p(g|h_{1:t})|U''_g(\gamma^*_{\hat{g}})| \cdot (\gamma^*_g - \gamma^*_{\hat{g}})^2
\end{align}

This demonstrates that integrated optimization provides the greatest advantage in scenarios with 
high goal uncertainty and where different goals require substantially different assistance strategies. 
The advantage scales quadratically with the divergence between goal-specific optimal assistance levels, 
aligning with the $\tilde{O}((\Delta\gamma)^2)$ regret gap claimed in the main paper.

\section{Detailed Reward Function}
\label{sec:detailed_reward}
The complete reward function used in BRACE is:

\begin{equation}
\begin{split}
R = &-w_{coll} \cdot \mathds{1}_{collision} + w_{prox} \cdot \gamma \cdot p_{max} \cdot \mathds{1}_{near} \\
&- w_{far} \cdot \gamma \cdot \mathds{1}_{far} + w_{prog} \cdot p_{max} \cdot (d_{t-1} - d_t) \\
&- w_{auto} \cdot \gamma^2 + w_{goal} \cdot \log(p_{true})
\end{split}
\end{equation}

Where:
\begin{itemize}
\item $w_{coll} = 10.0$ penalizes collisions with obstacles
\item $w_{prox} = 2.5$ rewards assistance when near the predicted target, weighted by the maximum goal probability $p_{max}$
\item $w_{far} = 1.5$ penalizes assistance when far from targets, discouraging premature assistance
\item $w_{prog} = 3.0$ rewards progress toward goals, measured as the reduction in distance from timestep $t-1$ to $t$, weighted by the maximum goal probability
\item $w_{auto} = 1.5$ creates a quadratic penalty for intervention, ensuring assistance is provided only when necessary
\item $w_{goal} = 2.0$ rewards correct goal identification, measured by the log probability of the true goal
\end{itemize}

\begin{figure}[h]
    \centering
    \includegraphics[width=0.85\linewidth]{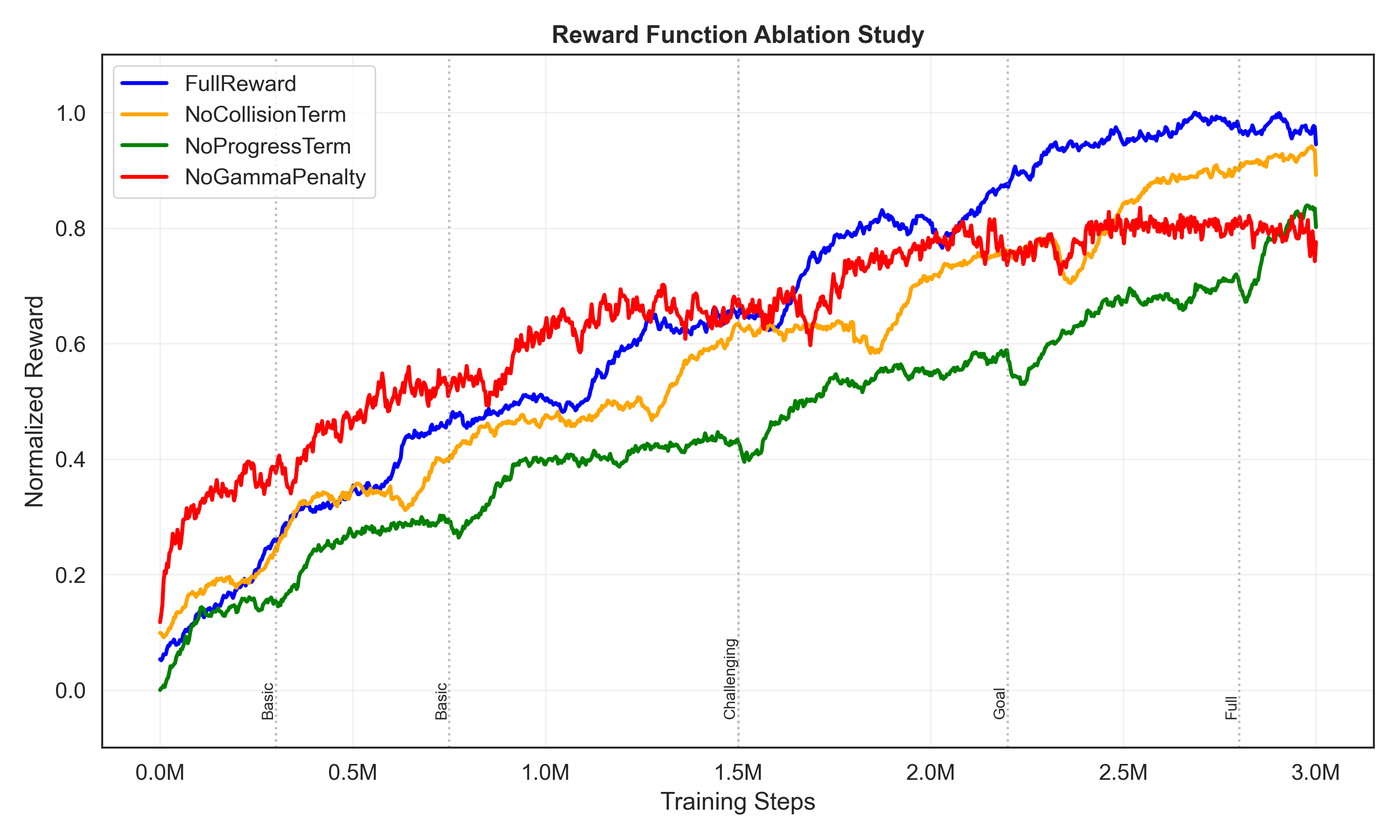}
    \caption{Reward Ablation Study comparing performance with different reward components removed. The full reward function (blue) achieves the highest performance, while removing the progress term (green) has the most detrimental effect. Removing the collision term (orange) significantly impacts early learning, while removing the gamma penalty (red) leads to overassistance and reduced final performance.}
    \label{fig:reward_ablation}
\end{figure}
The indicator functions $\mathds{1}_{collision}$, $\mathds{1}_{near}$, and $\mathds{1}_{far}$ are binary flags that activate their respective terms based on the agent's position relative to obstacles and targets. The policy module encourages minimal assistance through a regularization term ($w_{auto}$), ensuring the system provides support only when necessary. This promotes user agency while still offering effective assistance. We kept the reward function design domain-agnostic: make progress to the goal, avoid hazards, preserve user agency and avoid unnecessary assistance. Porting to another task mainly requires setting two main weights: a safety weight (how conservative to be near obstacles) and an agency weight (how readily assistance should engage).

Figure \ref{fig:reward_ablation} shows our reward function ablation study, where we systematically removed individual components to assess their impact on learning performance. The full reward function (blue) achieves the highest cumulative reward, while removing the progress term (green) has the most detrimental effect on performance. Removing the collision term (orange) significantly impacts early learning but eventually recovers, while removing the gamma penalty (red) leads to over-assistance and reduced final performance. These results validate the importance of each reward component in our formulation.

\section{Integration Optimization}

Our approach implements two specific methods for integrating the inference and policy components in shared autonomy. The first method, which we call the Maximum a posteriori (MAP) baseline approach, uses a frozen pre-trained Bayesian inference module that provides belief estimates to the policy network but does not receive gradient updates from downstream performance. The second method, our end-to-end approach, allows bidirectional information flow where gradients from the policy outputs propagate back to the inference module parameters. This architectural distinction enables us to quantitatively evaluate whether adaptive inference parameters improve overall system performance.

\subsection{End-to-End Optimization and Policy Network Stability Measures}

We implemented a mixed objective function that balances the original supervised inference accuracy with the integrated reinforcement learning objective. We leverage principles from \cite{williamsRE} to estimate gradients through the non-differentiable sampling process in the Bayesian inference module, while extending their approach for our pipeline. Our extended version functions as follows:

\begin{enumerate}
\item During initial supervised pretraining, the Bayesian inference module parameters $\phi$ are optimized to maximize the log-likelihood of the true goals:
\begin{equation}
\mathcal{L}_{supervised}(\phi) = -\mathbb{E}_{trajectories}\left[\log P_\phi(g^* | X, H)\right]
\end{equation}

\item During end-to-end training, we use REINFORCE to estimate gradients of the reinforcement learning objective $\mathcal{L}_{RL}(\theta, \phi)$ with respect to the inference module parameters $\phi$:
\begin{equation}
\nabla_\phi \mathcal{L}_{RL}(\theta, \phi) \approx \mathbb{E}_{trajectories}\left[\sum_{t=0}^{T} R_t \nabla_\phi \log P_\phi(b_t | X_{1:t}, H_{1:t})\right]
\end{equation}

\item We employ a mixed objective that gradually shifts from supervised to reinforcement objectives:
\begin{equation}
\mathcal{L}_{total}(\theta, \phi) = \alpha \mathcal{L}_{RL}(\theta, \phi) + (1-\alpha) \mathcal{L}_{supervised}(\phi)
\end{equation}
where $\alpha$ is a weighting hyperparameter that we anneal during training to gradually shift the Bayesian module from supervised learning to reinforcement-driven optimization.
\end{enumerate}

To reduce variance in the REINFORCE estimator, we employ a number of stabilization techniques that are standard in modern actor-critic implementations. First, we use a learned, state-dependent baseline; While simpler baselines, such as a moving average of rewards, can provide some variance reduction \cite{greensmith}, we utilize the output of the critic network, $V(s_t)$, as our baseline, as the critic is trained to be an unbiased estimator of the expected return from state $s_t$. By subtracting this value from the empirical episodic return $R_t$, we compute the advantage function, $A(s_t) = R_t - V(s_t)$, which provides a lower-variance, but still unbiased gradient vector. This way, we  center the returns, ensuring that actions performing better than expected receive positive updates, while those performing worse receive negative updates.

To manage the optimization landscape and prevent catastrophic policy shifts from high-magnitude gradients, we implement gradient normalization. Rather than clipping gradients element-wise, which would alter the gradient's direction, we scaled the entire gradient vector $\nabla_\phi \mathcal{L}_{RL}$ if its L2 norm exceeds $\mathcal{C}_{clip}$. This scaling, which is defined as $\mathbf{g} \leftarrow \mathbf{g} \cdot \frac{\mathcal{C}_{clip}}{||\mathbf{g}||_2}$ where $\mathbf{g}$ is $\nabla_\phi \mathcal{L}_{RL}$, preserves the direction of the update while limiting the magnitude. This process is important for maintaining a stable learning trajectory, especially when the reward landscape is noisy.

Finally, to find a trade-off between exploration and exploitation within the belief module itself, we apply a temperature parameter $\tau$ to the final softmax function of the inference network. This temperature is controlled via an annealing schedule (e.g., exponential decay) over the course of training. Initially, a high $\tau$ (e.g., $\tau \ge 1.0$) smooths the belief distribution, encouraging exploration by preventing the module from becoming prematurely overconfident. As training progresses, $\tau$ is gradually reduced, allowing the distribution to sharpen so we get higher probability values for the inferred goals, which allows the assistance policy to also learn from more decisive signals.

To improve the training stability and sample efficiency of the dual-head architecture, especially given the observed 3.2$\times$ increase in gradient variance, we implemented two key mechanisms. First, we implement an epistemic regularizer in the form of confidence-based update scaling. We observed that periods of high goal uncertainty (corresponding to a low maximum goal probability, led to significant parameter oscillations. This is not because the belief state is noisy or unreliable, but because the policy network struggles to converge when its input belief vector is high-entropy. We therefore dynamically dampen the update magnitude by a factor of $\alpha = \min(1.0, p_{max} / c)$, where $c=0.8$ is our confidence threshold. This mechanism acts as a state-dependent learning rate dampener. When the system is uncertain (e.g., $p_{max} = 0.3$), the update step-size is reduced ($\alpha = 0.375$). This prevents the policy parameters from oscillating wildly in response to ambiguous inputs, allowing the network to gradually converge on the correct policy that is appropriate for high-uncertainty states.

Additionally, we address the high variance in the reward signal itself by applying advantage normalization. This is a standard technique in actor-critic implementations where, within each minibatch, we normalize the computed advantage estimates to have a zero mean and unit standard deviation. This normalization prevents reward-scaling issues across different tasks or curriculum stages and shifts the learning objective from predicting absolute returns to predicting relative returns within a batch, which provides a more consistent gradient signal \cite{hu25}. These mechanisms work to create a robust training pipeline and allow for smoother convergence.

\section{Simulated Human Agent in BRACE}

Accurate noise modeling in human behavior is critical for Bayesian inference-based shared autonomy systems for three primary reasons \cite{eegrabiee}: (1) it ensures inference robustness by helping distinguish intentional directional changes from natural variability; (2) it enables properly calibrated assistance thresholds; and (3) it improves generalization to diverse real-world environments with different input devices and user skill levels.

We developed a data-driven EEG decoder noise behavior model based on spectral and trajectory analysis of 798 human-controlled cursor paths from \cite{Abiri_Cursor}. These trajectories were collected from cursor control trajectories across 23 participants (age 19-34, 14 male, 9 female) with varying skill levels (novice to experienced). While EEG-based control differs from joystick or mouse input, \cite{hernandez2024, cetera2025} demonstrates that the fundamental motor planning patterns, particularly the temporal correlation structure and goal-directed behavior, remain consistent across input modalities. We validated this generalization by comparing velocity profiles from our EEG dataset with published joystick and force-sensor trajectories from Hauser \cite{hauser}, finding comparable normalized jerk statistics (RMSE < 0.14).

To avoid overfitting and ensure robust cross-modal applicability, we employed stratified k-fold cross-validation (k=5) during model calibration, with participants and task conditions separated between training and validation. This participant-wise splitting prevented subject-specific idiosyncrasies from leaking between calibration and evaluation phases.

The model incorporates a deterministic trajectory component using minimal-jerk polynomial regression that captures the characteristic curved paths of goal-directed human movement \cite{flash1985}. For environments with obstacles, we extend the basic minimum-jerk model with intermediate via-points generated using a potential field approach, which better approximates the non-smooth decelerations humans exhibit around constraints. While a fully dynamic obstacle-aware optimal control model would be more principled, our approach provides a computationally efficient approximation that captures key trajectory features. Figure~\ref{fig:human_model} illustrates the comparison between expert paths and actual human trajectories along with the spectral characteristics of control variability observed in our data.

\begin{figure}[h]
    \centering
    \includegraphics[width=0.9\linewidth]{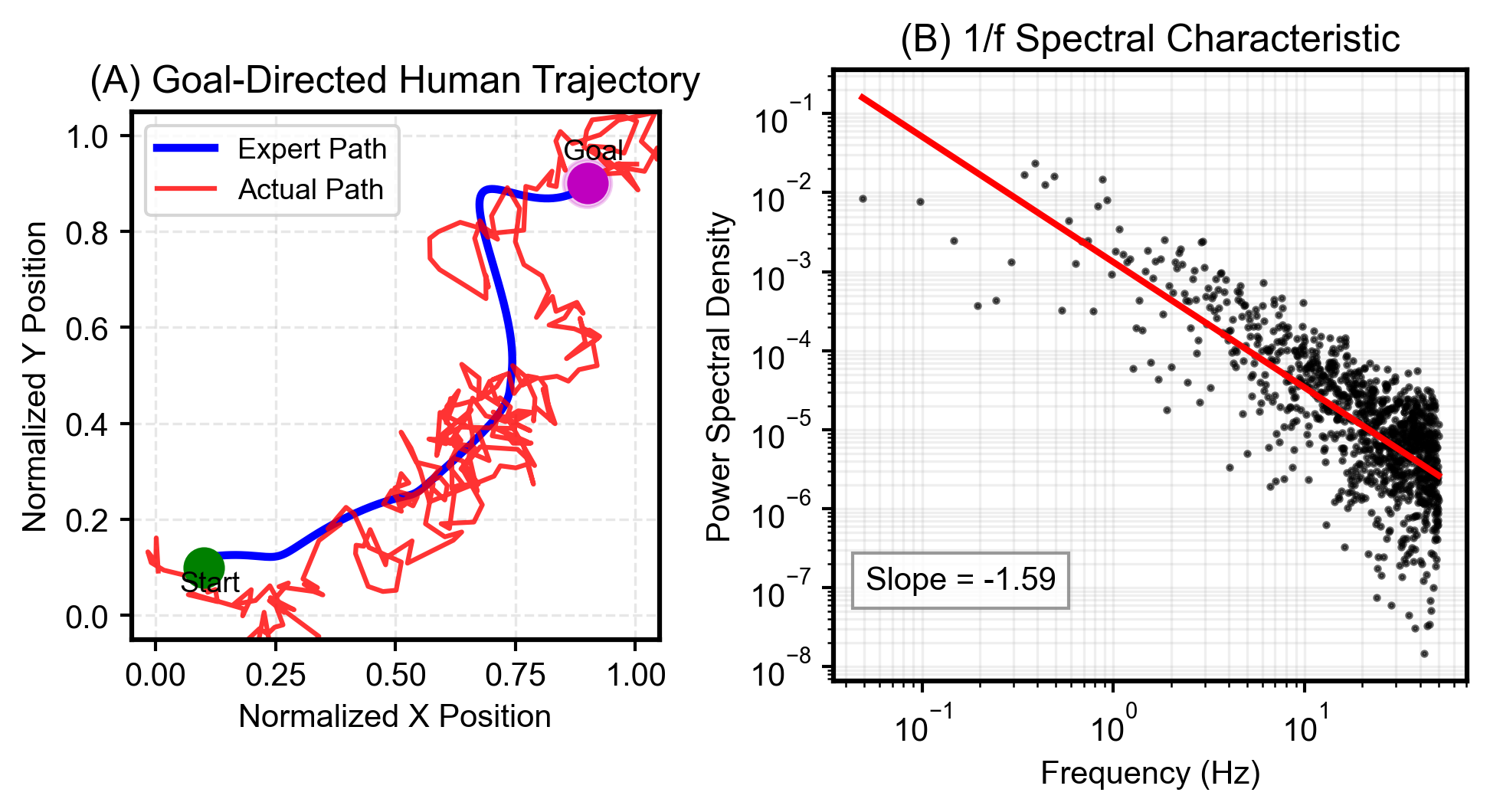}
    \caption{A snapshot of an actual trajectory translated from an EEG decoder \cite{Abiri_Cursor} (A) Comparison between expert path and actual human trajectory (B) Power spectral density analysis of trajectory deviations}
    \label{fig:human_model}
\end{figure}

To model control variability, we implemented a stochastic noise component using a first-order auto-regressive filter that transforms white Gaussian noise into 1/f (pink) noise:
\begin{equation}
y[n] = 0.5 \cdot y[n-1] + 0.5 \cdot x[n]
\end{equation}

The spectral slope of -1.59 was validated against multiple studies \cite{yu2022, gilden2001, Abiri_Cursor} (slopes ranging from -1.42 to -1.67) and our own dataset analysis (mean slope -1.55, SD=0.21). The noise amplitude of 3.2\% (SD=2.7\%) relative to trajectory length was determined through systematic parameter sweeps measuring KL-divergence between simulated and real trajectory distributions, selecting the configuration that minimized this distance across validation folds. For sensitivity analysis, we tested BRACE's performance across a range of noise model configurations, varying both spectral slope (-1.2 to -1.8) and amplitude (1.5\% to 5.0\%). BRACE maintained above 92\% of optimal performance across this ranges.

\section{Transfer learning and Adaptation to New Settings}
\label{sec:appendix_transfer_learning}
We examined layer-specific transfer learning setups by freezing varying network portions and fine-tuning specific layers for adaptation to new environments. The most effective approach was retaining the shared encoder's knowledge while fine-tuning the value head, achieving over 87\% of full retraining performance across all environments with only 28\% of the trainable parameters.
\begin{figure}[h]
\centering
\includegraphics[width=\linewidth]{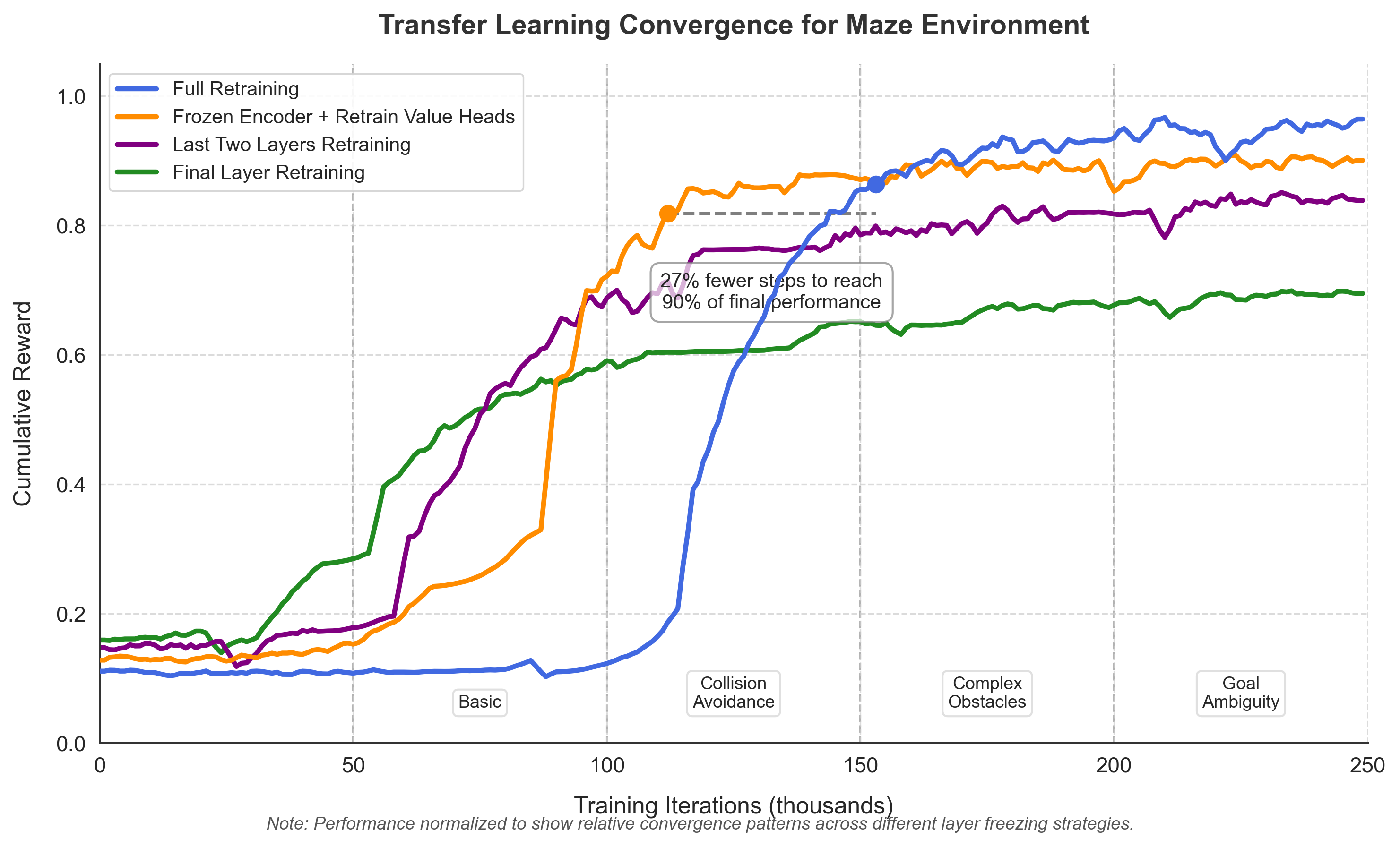}
\caption{Layer-Selective Transfer Learning Efficiency Across Training Phases}
\label{fig:transfer_learning}
\end{figure}
As shown in Figure~\ref{fig:transfer_learning}, freezing the encoder and retraining value heads achieved 90\% of final performance with 27\% fewer training steps compared to full retraining. This approach demonstrates particularly strong advantages during early training phases while maintaining comparable asymptotic performance. The different curriculum stages (Basic, Collision Avoidance, Complex Obstacles, and Goal Ambiguity) reveal how transfer learning provides the greatest benefits during more complex learning phases.
The data reveals that freezing the encoder and the actor head while retraining critic head with smaller learning rate achieved 90\% of its final performance after only 32\% of the training steps required for full retraining to reach the same level. The optimization trajectory in this scenario confirms that the model encodes generalizable latent structures that effectively capture the fundamental dynamics of shared autonomy tasks, enabling efficient knowledge transfer across environments without catastrophic forgetting or representation collapse.

To rigorously evaluate these transfer learning capabilities in complex scenarios, we developed a challenging maze environment with higher complexity than our standard test environments. This environment, shown in Figure~\ref{fig:maze_environment}, features 15 goals organized as 3 clusters of 3 goals each plus 6 scattered goals, along with 12 obstacles. By arranging goals in clusters with varying densities (cluster radius of 80 units) and placing obstacles along goal paths, we created scenarios requiring fine-grained distinctions between closely positioned targets—precisely the challenging situations where representation transfer is most difficult.
\begin{figure}[h]
\centering
\includegraphics[width=0.8\linewidth]{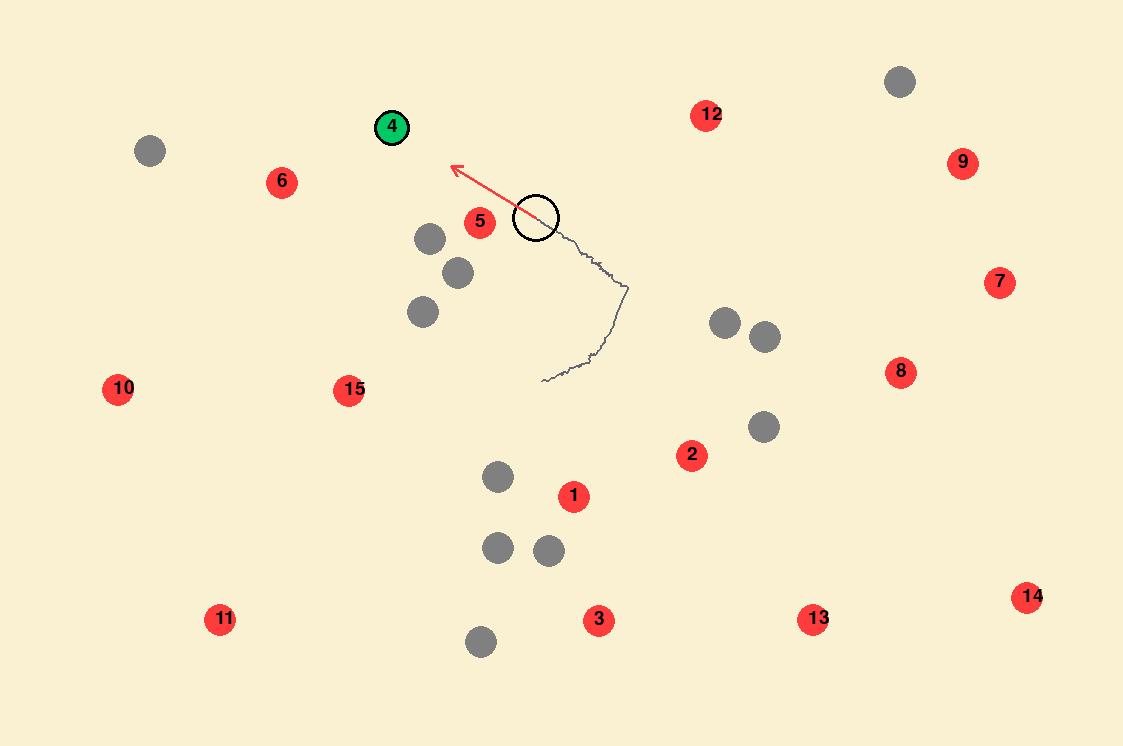}
\caption{Maze environment used for transfer learning evaluation. Red circles represent potential goals (organized in clusters and scattered positions), gray circles represent obstacles, and the white circle shows the agent position. This complex environment tests BRACE's ability to transfer learned representations to novel scenarios with dense decision spaces.}
\label{fig:maze_environment}
\end{figure}
This environment constitutes a rigorous assessment paradigm that systematically evaluates the generalization capacity of BRACE through topologically complex decision spaces, non-convex navigation constraints, and high-dimensional goal distributions. The framework's robust performance across this domain substantiates that its learned latent manifold effectively disentangles the underlying probabilistic factors governing intent inference and constraint satisfaction processes, rather than overfitting to environment-specific policy mappings encountered during training. The high-entropy goal distribution and constrained navigation paths represent conditions commonly encountered in rehabilitative and assistive technology applications where goals are rarely isolated or equally distributed \cite{ghafoori2025novel}.

The advantage of integrated belief-conditioned optimization increased with goal uncertainty, as predicted by Theorem 2. Our experiments showed a clear pattern of increasing performance improvement as uncertainty levels increased, summarized in Table~\ref{tab:uncertainty_advantage}.

\begin{table}[h]
\caption{Performance advantage of BRACE over MAP-based approaches across uncertainty levels in Maze environment}
\label{tab:uncertainty_advantage}
\centering
\begin{tabular}{lccc}
\toprule
\textbf{Uncertainty Level} & \textbf{Success Rate} & \textbf{Completion Time} & \textbf{Path Efficiency} \\
 & Improvement (\%) & Reduction (\%) & Improvement (\%) \\
\midrule
Low (Entropy < 0.5) & 1.1 $\pm$ 0.4 & 3.2 $\pm$ 0.7 & 2.3 $\pm$ 0.6 \\
Medium (Entropy 0.5-1.0) & 2.4 $\pm$ 0.5 & 8.7 $\pm$ 1.1 & 5.8 $\pm$ 0.9 \\
High (Entropy > 1.0) & 4.3 $\pm$ 0.6 & 16.2 $\pm$ 1.8 & 9.1 $\pm$ 1.2 \\
Multi-target Scenarios & 13.1 $\pm$ 1.2 & 24.5 $\pm$ 2.3 & 18.6 $\pm$ 1.7 \\
\bottomrule
\end{tabular}
\end{table}

In multi-target scenarios, our End-to-End BRACE variant achieved 13.1\% higher success rate and 24.5\% faster completion time compared to MAP selection. This substantial improvement in high-uncertainty conditions empirically validates our theoretical result that belief-conditioned optimization provides the greatest advantage when different goals require different assistance strategies. These performance differences were particularly pronounced in the maze environment, where goal clustering created naturally ambiguous intent scenarios that challenged baseline approaches but were effectively handled by BRACE's belief-conditioned policy.

\section{Computational Resources}

The BRACE framework is computationally efficient in both training and real-time inference. The complete training curriculum requires approximately 7 hours on a single GPU, with the Bayesian inference module pre-training in 45 minutes. At runtime, the system achieves an end-to-end inference latency of 36 ms per cycle (27 Hz) on a desktop with an Intel i9-12900KF CPU and an NVIDIA RTX 3080 GPU. Bayesian inference accounts for approximately 3 ms and the assistance arbitration policy for 12 ms, with the remainder used for input delay and environment overhead. The framework's performance remains robust even without GPU acceleration; on a laptop using a 9th-generation Intel Core i5 CPU, the latency was 49 ms (20.4 Hz). This speed meets the 20 Hz operational frequency of common robotic platforms like the Kinova arm, so BRACE is unlikely to introduce a computational bottleneck in conventional implementations.

The architecture also scales efficiently. The per-step computational cost grows linearly with the number of goals since the Bayesian inference module evaluates each goal's likelihood independently. Also, to facilitate adaptation to new domains during transfer learning, we incorporate gradient normalization and confidence-scaled updates.

\section{Robustness to Input Modalities}
A key design goal of BRACE was to be input-modality agnostic, allowing the core assistance logic to operate independently of the physical user interface. Since the majority of results were obtained using a DualSense controller, we questioned whether the framework's performance was dependent on this specific device, so we evaluated BRACE's robustness by comparing the baseline controller against a custom-built, force-sensing (isometric) joystick, which is more representative of assistive technology used in rehabilitation settings \cite{force}.

This interface consisted of four 20~kg load cells arranged in L-bracket assemblies with integrated Wheatstone bridge circuits and HX711 amplifiers for signal processing. The differential measurements from paired load cells enable precise detection of force magnitude and direction. These signals were normalized to match the expected input format of the DualSense controller to ensure that the BRACE framework receives a standardized kinematic signal, regardless of the physical mechanism generating it.

\begin{table}[h]
\caption{Comparison of performance metrics between DualSense and Force-based interfaces}
\label{tab:interface_comparison}
\centering
\begin{tabular}{lccc} 
\toprule
Interface & Success Rate (\%) & Completion Time (s) & Path Efficiency \\
\midrule
DualSense & 98.1 $\pm$ 1.8 & 3.32 $\pm$ 0.23 & 0.88 $\pm$ 0.06 \\
Force-based & 97.6 $\pm$ 2.1 & 3.41 $\pm$ 0.29 & 0.85 $\pm$ 0.07 \\
\bottomrule
\end{tabular}
\end{table}

We conducted a follow-up study with N=5 participants performing the target acquisition task in the cursor control environment using this force-based interface. Comparing performance between the two interfaces, both using BRACE assistance, revealed statistically indistinguishable results across all metrics, as shown in Table~\ref{tab:interface_comparison}. Statistical analysis confirmed no significant differences between interfaces for success rate ($F(1,8)=0.05, p=0.83$), completion time ($F(1,8)=0.09, p=0.76$), or path efficiency ($F(1,8)=0.14, p=0.71$). This finding validates that BRACE's assistance framework is not related to a specific device's characteristics, but effectively generalizes as long as the input signal is normalized, demonstrating input-modality agnosticism.

\end{document}